%% file: main.tex
\documentclass{article}

\usepackage[nonatbib, final]{neurips_2019}

\usepackage[utf8]{inputenc} 
\usepackage[T1]{fontenc}    
\usepackage{hyperref}       
\usepackage{url}            
\usepackage{booktabs}       
\usepackage{amsfonts}       
\usepackage{nicefrac}       
\usepackage{microtype}      
\usepackage{enumitem}

\usepackage{amsmath, amsthm,amssymb}

\newcommand{\cX}{\mathcal{X}}

\newcommand{\cE}{\mathcal{E}}
\newcommand{\cH}{\mathcal{H}}
\newcommand{\cN}{\mathcal{N}}
\newcommand{\cY}{\mathcal{Y}}
\newcommand{\cM}{\mathcal{M}}
\newcommand{\cO}{\mathcal{O}}
\newcommand{\EE}{\mathbb{E}}
\newcommand{\RR}{\mathbb{R}}
\newcommand{\eqdef}{\stackrel{\text{def}}{=}}
\newcommand{\indic}{\mathbf{1}}
\newcommand{\arxiv}[1]{}

\newcommand{\tcO}{\tilde{\mathcal{O}}}
\newtheorem{lemma}{Lemma}
\newtheorem{definition}{Definition}
\newtheorem{theorem}{Theorem}
\newtheorem{corollary}{Corollary}

\title{Differentially private anonymized histograms}

\author{
  Ananda Theertha Suresh \\
  Google Research, New York\\
  \texttt{theertha@google.com}
}

\begin{document}

\maketitle

\input{abstract}

\input{introduction}

\input{diff_privacy.tex}

\input{related.tex}

\input{algorithm_description}
\input{acknowledgements}
\newpage

\input{algorithm}
\bibliographystyle{unsrt} 
\bibliography{sample}

\newpage

\appendix

\begin{center}
{\Large Appendix: Differentially private anonymized histograms}    
\end{center}

\input{geometric_distributions.tex}

\input{distance_metric.tex}

\input{privacy_analysis.tex}

\input{utility_analysis.tex}

\input{lower_bound.tex}

\input{estimation_and_classification.tex}

\end{document}

%% file: abstract.tex
\begin{abstract}
For a dataset of label-count pairs, an anonymized histogram is the
multiset of counts. Anonymized histograms appear in various
potentially sensitive contexts such as password-frequency lists,
degree distribution in social networks, and estimation of symmetric
properties of discrete distributions. 
Motivated by these applications, we
propose the first differentially private mechanism to release anonymized
histograms that achieves near-optimal privacy utility
trade-off both in terms of number of items and the privacy parameter. 
Further, if the underlying histogram is given in a compact
format, the proposed algorithm runs in time sub-linear in the number
of items. For anonymized histograms generated from unknown discrete
distributions, we show that the released histogram can be directly
used for estimating symmetric properties of the underlying distribution.
\end{abstract}

%% file: introduction.tex
\section{Introduction}
Given a set of labels $\cX$, a dataset $D$ is a collection of labels
and counts, $D \eqdef \{(x, n_x) : x \in \cX\}$. 
An anonymized histogram of such a dataset is the unordered multiset of all
non-zero counts without any label information,
\[
h(D) \eqdef \{n_x : x \in \cX \text{ and } n_x > 0 \} .
\]
For example, if $\cX = \{a, b, c, d\}$, $D = \{(a, 8), (b, 0), (c, 8),
(d,3)\}$, then $h(D) = \{3, 8, 8\}$\footnote{$h(D)$ is a multiset and not a set and duplicates are
  allowed.}. Anonymized histograms
do not contain any information about the labels, including the
cardinality of $\cX$. Furthermore, we only consider histograms with
positive counts. The results can be extended to histograms that
include zero counts. A histogram can also be represented succinctly
using prevalences. For a histogram $h$, the \emph{prevalence}
$\varphi_r$ is the number of elements in the histogram
with count $r$,
\[
\varphi_r(h) \eqdef \sum_{n_x \in h} \indic_{n_x = r}.
\]
In the above example, $\varphi_3(h) = 1$, $\varphi_8(h) = 2$, and
$\varphi_r(h) = 0$ for $r \notin \{3, 8\}$. Anonymized histograms are
also referred to as histogram of histograms~\cite{batu2000testing},
histogram order statistics~\cite{paninski2003estimation},
{profiles}~\cite{orlitsky2004modeling}, {unattributed
  histograms}~\cite{hay2010boosting},
fingerprints~\cite{valiant2011estimating}, and frequency
lists~\cite{blocki2016differentially}.

Anonymized histograms appear in several potentially sensitive
contexts ranging from password
frequency lists to social networks. Before we proceed to the problem
formulation and results, we first provide an overview of the various contexts where
anonymized histograms have been studied under differential privacy and
their motivation.

\noindent \textbf{Password frequency lists}: Anonymized password
histograms are useful to security researchers who wish to understand
the underlying password distribution in order to estimate the security
risks or evaluate various password defenses~\cite{bonneau2012science,
  blocki2016differentially}. For example, if $n_{(i)}$ is the
$i^{\text{th}}$ most frequent password, then $\lambda_{\beta} =
\sum^\beta_{i=1} n_{(i)}$ is the number of
accounts that an adversary could compromise with $\beta$ guesses per
user. Hence, if the server changes the $k$-strikes policy from $3$ to
$5$, the frequency distribution can be used to evaluate the security
implications of this change. We refer readers to
\cite{blocki2016differentially, blocki2018economics} for more uses of password frequency
lists. Despite their usefulness, organizations may be wary of
publishing these lists due to privacy concerns. This is further justified as
it is not unreasonable to expect that an adversary will have some side
information based on attacks against other organizations. Motivated by
this,~\cite{blocki2016differentially} studied the problem of releasing
anonymized password histograms.

\noindent \textbf{Degree-distributions in social networks}: Degree
distributions is one of the most widely studied properties of a graph
as it influences the structure of the graph. Degree distribution can
also be used to estimate linear queries on degree distributions such
as number of $k$-stars. However, some graphs may have unique degree
distributions and releasing exact degree distributions is no more
safer than naive anonymization, which can leave social network participants vulnerable to a
variety of attacks~\cite{backstrom2007wherefore, hay2008resisting, narayanan2009anonymizing}. 
Thus releasing them exactly can be
revealing. Hence, 
\cite{hay2009accurate, hay2010boosting,
  karwa2012differentially, kasiviswanathan2013analyzing,
  raskhodnikova2015efficient, blockidifferentially2, day2016publishing} considered the problem of releasing
degree distributions of graphs with differential privacy. Degree
distributions are anonymized histograms over the graph
node degrees.

\noindent \textbf{Estimating symmetric properties of discrete
  distributions}: Let $k \eqdef |\cX|$. A discrete distribution $p$ is
a mapping from a domain $\cX$ to $[0,1]^{k}$ such that $\sum_x p_{x} =
1$. Given a discrete distribution $p$ over $k$ symbols, a
\emph{symmetric property} is a property that depends only on the
multiset of probabilities~\cite{valiant2011power, acharya2017unified},
e.g., entropy ( $ \sum_{x \in \cX} p_x \log \frac{1}{p_x}$).  Other
symmetric properties include support size, R\'enyi entropy, distance
to uniformity, and support coverage.  Given independent samples from
an unknown $p$, the goal of property estimation
is to estimate the value of the symmetric property of interest for
$p$. Estimating symmetric properties from unknown distributions has
received a wide attention in the recent past
e.g.,~\cite{valiant2011estimating, valiant2011power, jiao2015minimax,
  wu2016minimax, orlitsky2016optimal,acharya2017unified, yi2018data, hao2019data} and has
applications in various fields from
neuro-science~\cite{paninski2003estimation} to
genetics~\cite{zou2016quantifying}. Recently,~\cite{acharya2018inspectre}
proposed algorithms to estimate support size, support coverage and
entropy with differential privacy. Optimal estimators for symmetric
properties only depend on the anonymized histograms of the
samples~\cite{batu2000testing, acharya2017unified}. Hence, releasing
anonymous histograms with differential privacy would simultaneously
yield differentially-private plug-in estimators for all symmetric
properties.

%% file: diff_privacy.tex
\section{Differential privacy}
\label{sec:privacy}
\subsection{Definitions}
Before we outline our results, we first define the privacy and utility
aspects of anonymized histograms. Privacy has been studied extensively
in statistics and computer science~\cite{dalenius1977towards,
  adam1989security, agrawal2001design, dwork2008differential} and
references therein. Perhaps the most studied form of privacy is
differential privacy (DP)~\cite{dwork2006calibrating,
  dwork2014algorithmic}, where the objective is to
ensure that an adversary would not infer whether a user is present in the
dataset or not.

We study the problem of releasing anonymized histograms via the lens
of global-DP. We begin by defining the notion of
DP. Formally, given a set of datasets $\cH$ and a
notion of neighboring datasets $\cN_{\cH} \subseteq \cH \times \cH$,
and a query function $f: \cH \to \cY$, for some domain $\cY$, then a
mechanism $\cM: \cY \to \cO$ is said to be $\epsilon$-DP, 
if for any two neighboring datasets $(h_1, h_2) \in
\cN_{\cH}$, and all $S \subseteq \cO$,
\begin{equation}
\label{eq:diff_privacy}
\Pr(\cM(f(h_1)) \in S) \leq e^{\epsilon} \Pr(\cM(f(h_2)) \in S).
\end{equation}
Broadly-speaking $\epsilon$-DP ensures that given
the output, an attacker would not be able to differentiate between any
two neighboring datasets. $\epsilon$-DP is also
called \emph{pure}-DP and provides stricter
guarantees than the approximate $(\epsilon, \delta)$-DP,
where equation~\eqref{eq:diff_privacy} needs to hold with
probability $1-\delta$.

Since introduction, DP has been studied
extensively in various applications from dataset release to learning
machine learning models~\cite{abadi2016deep}. It has also been adapted
by industry~\cite{erlingsson2014rappor}. There are two models of
DP: \emph{server} or \emph{global} or \emph{output}
DP, where a centralized entity has access to the
entire dataset and answers the queries in a DP
manner. The second model is \emph{local} DP, where
$\epsilon$-DP is guaranteed for each individual
user's data~\cite{warner1965randomized,
  kasiviswanathan2011can, duchi2013local, kairouz2014extremal,
  acharya2018communication}. We study the problem of releasing anonymized histograms 
  under global-DP. Here $\cH$ is the set of anonymized histograms,
$f$ is the identity mapping, and $\cO = \cH$.

\subsection{Distance measure}
For DP, a general notion of neighbors is as
follows. Two datasets are neighbors if and only if one can be obtained
from another by adding or removing a
user~\cite{dwork2008differential}. Since, anonymized histograms do not
contain explicit user information, we need few definitions to apply
the above notion. We first define a notion of distance between
label-count datasets. A natural notion of distance between datasets
$D_1$ and $D_2$ over $\cX$ is the $\ell_1$ distance,
\[
\ell_1(D_1, D_2) \eqdef \sum_{x \in \cX} |n_x(D_1) - n_x(D_2)|,
\]
where $n_x(D)$ is the count of $x$ in dataset $D$. Since anonymized
histograms do not contain any information about labels, we define
distance between two histograms $h_1 , h_2$ as
\begin{equation}
\label{eq:distance}
\ell_1(h_1, h_2) \eqdef \min_{D_1, D_2 : h(D_1) = h_1, h(D_2) = h_2} \ell_1(D_1, D_2).
\end{equation}
The following simple lemma characterizes the above distance in terms
of counts. 
\begin{lemma}[Appendix~\ref{app:distance}]
\label{lem:distance}
For an anonymized histogram $h = \{n_x\}$, let $n_{(i)}$ be the
$i^{\text{th}}$ highest count in the dataset.\footnote{For $i$ larger
  than number of counts in $h$, $n_{(i)} = 0$.} For any two
anonymized histograms $h_1,h_2$,
$
\ell_1(h_1, h_2) =  
\sum_{i\geq 1} |n_{(i)}(h_1) - n_{(i)}(h_2)|.
$
\end{lemma}
The above distance is also referred to as sorted $\ell_1$ distance or
earth-mover's distance. 
With the above definition of distance, we can
define neighbors as follows.
\begin{definition}
\label{def:neighbor}
Two anonymized histograms $h$ and $h'$ are neighbors if and only if
$\ell_1(h,h') = 1$.
\end{definition}
The above definition of neighboring histograms is same as the
definition of neighbors in the previous works on anonymized histograms~\cite{hay2010boosting,
  blocki2016differentially}. 

%% file: related.tex
\section{Previous and new results}
\label{sec:results}
\subsection{Anonymized histogram estimation}
Similar to previous works~\cite{blocki2016differentially}, we measure
the utility of the algorithm in terms of the number of items in
the anonymized histogram, $n \eqdef \sum_{n_x \in h} n_x =
\sum_{r \geq 1} \varphi_r(h) r $.

\noindent\textbf{Previous results}: The problem of releasing
anonymized histograms was first studied
by~\cite{hay2009accurate,hay2010boosting} in the context of degree
distributions of graphs. They showed that adding Laplace noise to
each count, followed by a post-processing isotonic regression step
results in a histogram $H$ with expected sorted-$\ell^2_2$ error of
\[
\EE[\ell^2_2(h, H)] = \EE \left[\sum_{i\geq 1} (n_{(i)}(h_1) - n_{(i)}(h_2))^2 \right] = \sum_{r \geq 0} \cO \left(\frac{\log^3
  \max(\varphi_r, 1)}{\epsilon^2}\right) = \cO
\left(\frac{\sqrt{n}}{\epsilon^2}\right).
\]
Their algorithm runs in time $\cO(n)$. 
The problem was also considered in the context of password frequency lists
by~\cite{blocki2016differentially}.  They observed that an exponential
mechanism over integer partitions yields an $\epsilon$-DP algorithm. Based on this, for $\epsilon = \Omega(1/\sqrt{n})$,
they proposed a dynamic programming based relaxation of the
exponential mechanism that runs in time $\cO
\left(\frac{n^{3/2}}{\epsilon} +n \log \frac{1}{\delta}\right)$ and
returns a histogram $H$ such that
$
\ell_1(h, H) = \cO \left(\frac{\sqrt{n} + \log
  \frac{1}{\delta}}{\epsilon} \right),
$
with probability $\geq 1 - \delta$. Furthermore, the relaxed mechanism
is $(\epsilon, \delta)$-DP.

The best information-theoretic lower bound for the $\ell_1$ utility of
any $\epsilon$-DP mechanism is due to
\cite{alda2018lower}, who showed that for $\epsilon \geq \Omega(1/n)$, any
$\epsilon$-DP mechanism has expected $\ell_1$ error
of $\Omega(\sqrt{n}/\sqrt{\epsilon})$ for some dataset.

\noindent\textbf{New results}: Following~\cite{blocki2016differentially}, we study the
problem in $\ell_1$ metric. We propose a new DP
mechanism \textsc{PrivHist} that satisfies the following:
\begin{theorem}
\label{thm:main}
Given a histogram in the prevalence form $h = \{(r, \varphi_r) :
\varphi_r > 0\}$, \textsc{PrivHist} returns a histogram $H$ and a
sum count $N$ that is $\epsilon$-DP. Furthermore, if $\epsilon > 1$, then
\begin{equation*}
\EE[\ell_1(h,H)] = \cO \left(\sqrt{n} \cdot e^{-c\epsilon} \right) \text{~~~and~~~} \EE[|N - n|] \leq e^{-c\epsilon}
\end{equation*}
for some constant $c > 0$ and has an expected run time of $\tcO(\sqrt{n})$. If $1 \geq \epsilon = \Omega(1/n)$
then,
\begin{align*}
\EE[\ell_1(h,H)] = \cO \left(\sqrt{\frac{n}{\epsilon} \log
  \frac{2}{\epsilon}}\right) \text{~~~and~~~} \EE[|N - n|] \leq \cO
\left(\frac{1}{\epsilon} \right),
\end{align*}
and has an expected run time of $\tcO\bigl(\sqrt{\frac{n}{\epsilon}} + \frac{1}{\epsilon}\bigr)$.
\end{theorem}
Together with the lower bound of \cite{alda2018lower}, this settles
the optimal privacy utility trade-off for $\epsilon \in [\Omega(1/n),
  1]$ up to a multiplicative factor of $\cO(\sqrt{\log (2/\epsilon)})$. We also
show that \textsc{PrivHist} is near-optimal for $\epsilon > 1$, by
showing the following lower bound.
\begin{theorem}[Appendix~\ref{app:low_high_privacy}]
\label{thm:low_high_privacy}
For a given $n$, let 
$\cH = \{h : n \leq \sum_{r} r \varphi_r(h) \leq n + 1 \}$. For any $\epsilon$-DP mechanism $\cM$, there exists a histogram $h \in \cH$, such that
\[
\EE[\ell_1(h, \cM(h))] \geq \Omega(\sqrt{n} e^{-2\epsilon}).
\]
\end{theorem}
Theorems~\ref{thm:main} and \ref{thm:low_high_privacy} together with
\cite{alda2018lower} show that the the proposed mechanism has
near-optimal utility for all $\epsilon = \Omega(1/n)$. We
can infer the number of items in the dataset by $\sum_{r} r \cdot
\varphi_r(H)$. However, this estimate is very noisy. Hence, we also
return the sum of counts $N$ as it is useful for applications in
symmetric property estimation for distributions. 
Apart from the near-optimal
privacy-utility trade-off, we also show that \textsc{PrivHist} has
several other useful properties.

\noindent \textbf{Time complexity}: By the Hardy-Ramanujan integer partition theorem \cite{hardy1918asymptotic}, the number of anonymized histograms with $n$ items is $e^{\Theta(\sqrt{n})}$. Hence, we can succinctly represent them using $\cO(\sqrt{n})$ space. Recall that any anonymized
histogram can be written as $\{(r,
\varphi_r) : \varphi_r > 0\}$, where $\varphi_r$ is the number of
symbols with count $r$. Let $t$ be the number of distinct counts and
let $r_1, r_2, \ldots, r_t$ be the distinct counts with non-zero
prevalences. Then $r_i \geq i$ and
\begin{align*}
n  = \sum^t_{i=1} r_i \varphi_{r_i} \geq  \sum^t_{i=1}  r_i \geq  \sum^t_{i=1}  i \geq  \frac{t^2}{2},
\end{align*}
and hence there are at most $t \leq \sqrt{2n}$ non-zero prevalences
and $h$ can be represented as $\{(r,\varphi_r) : \varphi_r > 0\}$
using $\cO(\sqrt{n})$ count-prevalence pairs. Histograms
are often stored in this format for space efficiency e.g., password frequency lists in \cite{bonneau_2015}.
\textsc{PrivHist} takes advantage of this succinct
representation. Hence, given such a succinct representation, it runs
time $\cO(\sqrt{n})$ as opposed to the $\cO(n)$ running time of
$\cite{hay2009accurate}$ and $\cO(n^{3/2})$ running time of
~\cite{blocki2016differentially}. This is highly advantageous for
large datasets such as password frequency lists with $n = 70M$ data
points~\cite{blocki2016differentially}.

\noindent \textbf{Pure vs approximate differential
  privacy}: The only previous known algorithm with $\ell_1$ utility of  $\cO(\sqrt{n})$
  is that of \cite{blocki2016differentially} and it
runs in time $\cO(n^{3/2})$. However, their algorithm is $(\epsilon,
\delta)$-approximate DP which is strictly weaker
than \textsc{PrivHist}, whose output is $\epsilon$-DP. For applications in social networks it is desirable to have
group privacy for large groups \cite{dwork2014algorithmic}. For groups
of size $k$, $(\epsilon, \delta)$ approximate DP, scales as $(k\epsilon, e^{k\epsilon} \delta)$-DP, which can be prohibitive for large values of $k$. Hence $\epsilon$-DP is preferable.

\noindent \textbf{Applications to symmetric property
  estimation}: We show that the output of \textsc{PrivHist} can be
directly applied to obtain near-optimal sample complexity algorithms
for discrete distribution symmetric property estimation. 

\subsection{Symmetric property estimation of discrete distributions}
For a symmetric property $f$ and an estimator $\hat{f}$ that uses $n$
samples, let $\cE(\hat{f}, n)$ be an upper bound on the worst expected
error over all distributions $p$ with support at most $k$, $
\cE(\hat{f}, n) \eqdef \max_{p \in \Delta^k} \EE[|f(p) - \hat{f}(X^n)|]$
. Let sample complexity $n(f,\alpha)$ denote the minimum number of samples
such that $\cE(\hat{f}, n) \leq \alpha$,
$
n(f,\alpha) \eqdef \min \{n : \cE(\hat{f}, n) \leq \alpha\}.
$

Given samples $X^n \stackrel{\text{def}}{=} X_1, X_2, \ldots, X_n$, let $h(X^n)$ denote the corresponding anonymous histogram. For a symmetric property $f$, linear estimators of the form 
\[
\hat{f}(h(X)) \eqdef \sum_{r\geq 1} f(r,n)\cdot \varphi_r(h(X^n),
\]
are shown to be sample-optimal for symmetric properties such as entropy~\cite{wu2016minimax}, support size~\cite{valiant2011power,jiao2015minimax}, support coverage~\cite{orlitsky2016optimal}, and R\'enyi entropy \cite{acharya2014complexity, acharya2016estimating}, where $f(r,n)$s are some distribution-independent coefficients that depend on the property $f$. Recently, \cite{acharya2018inspectre} showed that for any given property such as entropy or support size, one can construct DP estimators by adding Laplace noise to the non-private estimator. They further showed that this approach is information theoretically near-optimal.

Instead of just computing a DP estimate for a given property, the output of \textsc{PrivHist} can be directly used to estimate any symmetric property. By the post-processing lemma \cite{dwork2014algorithmic}, since the output of \textsc{PrivHist} is DP, the estimate is also DP. For an estimator $\hat{f}$, let $L^n_{\hat{f}}$ be the Lipschitz constant given by 
$
L^n_{\hat{f}} \stackrel{\text{def}}{=} \max(f(1,n), \max_{r \geq 1} |f(r,n)- f({r+1}, n)|). 
$
If instead of $h(X^n)$, a DP histogram $H$ and the sum of counts $N$ is available, then $\hat{f}$ can be modified as 
\[
\hat{f}^{dp} \eqdef \sum_{r\geq 1} f(r,N) \cdot \varphi_r(H),
\]
which is differentially private. Using Theorem~\ref{thm:main}, we show that:
\begin{corollary}[Appendix~\ref{app:corollary}]
\label{cor:properties}
Let $\hat{f}$ satisfy $L^n_{\hat{f}} \leq n^{\beta-1}$, for a $\beta \leq 0.5$. Further, let there exists $\cE$ such that $|\cE(\hat{f},n) - \cE(\hat{f}, n+1)| \leq n^{\beta-1}$. Let $f_{\max} = \max_{p \in \Delta^k} f(p)$. If $n(\hat{f},\alpha)$ is the sample complexity of estimator $\hat{f}$, then for $\epsilon > 1$
\[
n(\hat{f}^{dp}, 2\alpha) \leq \max\left( n(\hat{f}, \alpha),  \cO \left(\left( \frac{1}{\alpha e^{c\epsilon} }
\right)^{\frac{2}{1-2\beta}} +  \frac{1}{\epsilon} \log \frac{f_{\max}}{\alpha} \right) \right).
\]
for some constant $c > 0$. For $\Omega(1/n) \leq \epsilon \leq 1$,
\[
n(\hat{f}^{dp}, 2\alpha) \leq  \max\left( n(\hat{f}, \alpha), \cO \left( \left( \frac{\log (2/\epsilon)}{\alpha^2 \epsilon}
\right)^{\frac{1}{1-2\beta}} + \frac{1}{\epsilon} \log \frac{f_{\max}}{\alpha \epsilon} \right) \right).
\]
Further, by the post-processing lemma, $\hat{f}^{dp}$ is also $\epsilon$-DP.
\end{corollary}
For entropy ($-\sum_x p_x \log p_x$), normalized support size ($\sum_x \indic_{p_x > 1/k}/k$), and normalized support coverage, there exists sample-optimal linear estimators with $\beta < 0.1$ and have the  property $|\cE(\hat{f},n) - \cE(\hat{f}, n+1)| \leq \cE(\hat{f}, n)n^{\beta-1}$~\cite{acharya2017unified, acharya2018inspectre}. Hence the sample complexity of the proposed algorithm increases at most by a polynomial in $1/\epsilon\alpha$. 
Furthermore, the increase is dependent on the maximum value of the function for distributions of interest and it does not explicitly depend on the support size.
This result is slightly worse than the property specific results of~\cite{acharya2018inspectre} in terms of dependence on $\epsilon$ and $\alpha$. In particular, for entropy estimation, the main term in our privacy cost is $\tilde{\mathcal{O}} \left(\left( {1}/{\alpha^2\epsilon}  \right)^{\frac{1}{1-2\beta}} \right)$ and the bound of \cite{acharya2018inspectre} is $\mathcal{O}\left(1/(\alpha\epsilon)^{1+\beta}\right)$. Thus for $\beta=0.1$, our dependence on $\epsilon$ and $\alpha$ is slightly worse.
However, we note that our results are more general in that $H$ can be used with any linear estimator. For example, our algorithm implies DP algorithms for estimating distance to uniformity, which have been not been studied before.
Furthermore, \textsc{PrivHist} can also be combined with the maximum likelihood estimators  of~\cite{orlitsky2004modeling,orlitsky2016optimal, hao2019broad} and linear programming estimators of~\cite{valiant2011estimating}, however we do not provide any theoretical guarantees for these combined algorithms.

%% file: algorithm_description.tex
\section{\textsc{PrivHist}}
\label{sec:algorithm}
In the algorithm description and analysis, let $\bar{x}$ denote the vector $x$ and let $\varphi_{r^+}(h) \eqdef \sum_{s \geq r } \varphi_s(h)$ denote the cumulative prevalences.
Since, anonymized histograms are multisets, we can
define the sum of two anonymized histograms as follows: for two
histograms $h_1, h_2$, the sum $h = h_1 + h_2$ is given by
$
\varphi_r(h) = \varphi_r(h_1) + \varphi_r(h_2), \forall r.
$
Furthermore, since there is a one-to-one mapping between histograms in count form $h = \{n_{(i)}\}$ and in prevalence form $h = \{(r, \varphi_r) : \varphi_r > 0\}$, we use both interchangeably. For the ease of analysis, we also use the notation of improper histogram, where the $\varphi_r$'s can be negative or non-integers. Finally, for a histogram $h^a$ indexed by super-script $a$ , we define $\varphi^a \stackrel{\text{def}}{=} \varphi(h^a)$
for the ease of notation.
\subsection{Approach}
Instead of describing the technicalities involved in the algorithm directly, we first motivate the algorithm with few incorrect or high-error algorithms.  Before we proceed, recall that histograms can be written either in terms of prevalences $\varphi_r$ or in terms of sorted counts $n_{(i)}$. 

\noindent\textbf{An incorrect algorithm}: A naive tendency would be to just add noise only to the non-zero prevalences or counts. However, this is not differentially private. For example, consider two neighboring histograms in prevalence format, $h = \{\varphi_1 = 2\}$ and $h' = \{\varphi_1=1, \varphi_2=1\}$. The resulting outputs for the above two inputs would be very different as the output of $h$ never produces a non-zero $\varphi_2$, whereas the output of $h'$ produces non-zero $\varphi_2$ with high probability. Similar counter examples can be shown for sorted counts.

\noindent\textbf{A high-error algorithm}: Instead of adding noise to non-zero counts or prevalences, one can add noise to all the counts or prevalences. It can be shown that adding noise to all the counts (including those appeared zero times), yields a $\ell_1$ error $\cO(n/\epsilon)$, whereas adding noise to prevalences can yield an $\ell_1$ error of  $\cO(n^2/\epsilon)$, if we naively use the utility bound in terms of prevalences~\eqref{eq:dd1}. 
We note that \cite{hay2009accurate} showed that a post-processing step after adding noise to sorted-counts and improves the $\ell_2$ utility. A naive application of the Cauchy-Schwarz inequality yields an $\ell_1$ error of $n^{3/4}/\epsilon$ for that algorithm. While it might be possible to improve the dependence on $n$ by a tighter analysis, it is not clear if the dependence on $\epsilon$ can be improved. 

The algorithm is given in \textsc{PrivHist}. After some computation, it calls two sub-routines \textsc{PrivHist-LowPrivacy} and \textsc{PrivHist-HighPrivacy} depending on the value of $\epsilon$. \textsc{PrivHist} has two main new ideas: (i) splitting $r$ around $\sqrt{n}$ and using prevalences in one regime and counts in another and (ii) the smoothing technique used to zero out the prevalence vector. Of the two (i) is crucial for the computational complexity of the algorithm and (ii) is crucial in improving the $\epsilon$-dependence from $1/\epsilon$ to $1/\sqrt{\epsilon}$ in the high privacy regime ($\epsilon \leq 1$). There are more subtle differences such as using cumulative prevalences instead of actual prevalences. We highlight them in the next section. 
 We now overview our algorithm for low and high privacy regimes separately.
\subsection{Low privacy regime}
We first consider the problem in the low privacy regime when $\epsilon > 1$. We make few observations.

\noindent\textbf{Geometric mechanism vs Laplace mechanism}: For obtaining DP output of integer data, one can add either Laplace noise or geometric noise \cite{ghosh2012universally}. For $\epsilon$-DP, the expected $\ell_1$ noise added by Laplace mechanism is $1/\epsilon$, which strictly larger than that of the geometric mechanism ($2e^{-\epsilon}/(1-e^{-2\epsilon})$) (see Appendix~\ref{app:moments}). For $\epsilon > 1$, we use the geometric mechanism to obtain optimal trade off in terms of $\epsilon$.

\noindent\textbf{Prevalences vs counts}: If we add noise to each coordinate of a $d$-dimensional vector, the total amount of noise in $\ell_1$ norm scales linearly in $d$, hence it is better to add noise to a small dimensional vector. In the worst case, both prevalences and counts can be an $n$-dimensional vector. Hence, we propose to use prevalences for small values of $r \leq \sqrt{n}$, and use counts for $r > \sqrt{n}$. This ensures that the dimensionality of vectors to which we add noise is at most $\cO(\sqrt{n})$.

\noindent\textbf{Cumulative prevalences vs prevalences}: The $\ell_1$ error can be bounded in terms of prevalences as follows. See Appendix~\ref{app:distance} for a proof.
\begin{align}
\ell_1(h_1, h_2) 
  \leq \sum_{r \geq 1} |\varphi_{r^+}(h_1) - \varphi_{r^+}(h_2)|
\leq \sum_{r\geq 1} r |\varphi_r(h_1) - \varphi_r(h_2)|
, \label{eq:dd1}
\end{align}
If we add noise to prevalences, the $\ell_1$ error can be very high as noise is multiplied with the corresponding count $r$~\eqref{eq:dd1} . The bound in terms of cumulative prevalences $\varphi_{+}$ is much tighter. Hence, for small values of $r$, we use cumulative prevalences instead of prevalences themselves.

The above observations provide an algorithm for the low privacy regime. However, there are few technical difficulties. For example, if we split the counts at a threshold naively, then it is not differentially private. We now describe each of the steps in the high-privacy regime.

\noindent\textbf{(1) Find $\sqrt{n}$}: To divide the histogram into two smaller histograms, we need to know $n$, which may not be available. Hence, we allot $\epsilon_1$ privacy cost to find a DP value of $n$.
\arxiv{Even though the algorithm is analyzed with $\epsilon_1 = \epsilon/2$, as analysis reveals, $\epsilon_1 = \epsilon /\sqrt{n}$ is a better choice. Hence, if all the histograms under consideration have a minimum value of $n$, say $n_{\min}(\cH)$ and the value of $N$ itself is not useful in future, then we suggest $\epsilon_1 = \epsilon /\sqrt{n_{\min(\cH)}}$ and $\epsilon_2 = \epsilon - \epsilon_1$.}

\noindent \textbf{(2) Sensitivity preserving histogram split}: If we divide the histogram into two parts based on counts naively and analyze the privacy costs independently for the higher and smaller parts separately, then the sensitivity would be lot higher for certain neighboring histograms. For example, consider two neighboring histograms $h_1 = \{\varphi_{T} = 1, \varphi_{n -T} = 1\}$ and $h_2 = \{\varphi_{T+1} = 1, \varphi_{n -T-1} = 1\}$. If we divide $h_1$ in to two parts based on threshold $T$, say $h^s_1 = \{\varphi_{T} = 1\}$ and $h^\ell_1 = \{\varphi_{n -T} = 1\}$ and $h^s_2 = \{\}$ and $h^\ell_2 =  \{\varphi_{T+1} = 1, \varphi_{n -T-1} = 1\}$, then $\ell_1(h^\ell_1, h^\ell_2) = T + 2$. Thus, the $\ell_1$ distance between neighboring 
separated histograms $\ell_1(h^\ell_1, h^\ell_2)$, $\ell_1(h^s_1, h^s_2)$ would be much higher compared to $\ell_1(h_2, h_2)$ and we need to add a lot of noise. Therefore, we perturb $\varphi_{T}$ and $\varphi_{T+1}$ using geometric noise. This ensures DP in instances where the neighboring histograms differ at $\varphi_{T}$ and $\varphi_{T+1}$, and doesn't change the privacy analysis for other types of histograms. 
However, adding noise may make the histogram improper as $\varphi_{T}$ can become negative. To this end, we add $M$ fake counts at $T$ and $T+1$ to ensure that the histogram is proper with high probability. We remove them later in \textbf{(L4)}. We refer readers to Appendix~\ref{app:split} for details about this step.

\noindent\textbf{(3,4) Add noise}: Let $H^{bs}$ (small counts) and $H^{b\ell}$ (large counts) be the split-histograms. We add noise to cumulative prevalences in $H^{bs}$  and counts in $H^{b\ell}$ as described in the algorithm overview.

\noindent \textbf{(L1, L2) Post-processing}: The noisy versions of $\varphi_{r^+}$ may not satisfy the properties satisfied by the histograms i.e., $\varphi_{r^+} \geq \varphi_{(r+1)^+}$. To overcome this, we run isotonic regression over noisy $\varphi_{r^+}
$ subject to the monotonicity constraints i.e., given noisy counts $\varphi_{r^+}$, find $\varphi^{\text{mon}}_{r^+}$ that minimizes $\sum_{r \leq T} (\varphi_{r^+} - \varphi^{\text{mon}}_{r^+})^2$,
subject to the constraint that $\varphi^{\text{mon}}_{r^+} \geq \varphi^{\text{mon}}_{(r+1)^+}$, for all $r \leq T$. Isotonic regression in one dimension can be run in time linear in the number of inputs using the pool adjacent violators algorithm (PAVA) or its variants~\cite{barlow1972statistical,mair2009isotone}. Hence, the time complexity of this algorithm is $\cO(T) \approx \sqrt{n}$. We then round the prevalences to the nearest non-negative integers. We similarly post-process large counts and remove the fake counts that we introduced in step \textbf{(2)}. 

Since we used succinct representation of histograms, used prevalences for $r$ smaller than $\cO(\sqrt{n})$ and counts otherwise, the expected run time of the algorithm is $\tcO(\sqrt{n})$ for $\epsilon > 1$.
\subsection{High privacy regime}
\label{sec:high_algorithm}
 For the high-privacy regime, when $\epsilon \leq  1$, all known previous algorithms achieve an error of $1/\epsilon$. To reduce the error from $1/\epsilon$ to $1/\sqrt{\epsilon}$, we use smoothing techniques to reduce the sensitivity and hence reduce the amount of added noise. 
 
 \noindent \textbf{Smoothing method}: Recall that the amount of noise added to a vector depends on its dimensionality. Since prevalences have length $n$, the amount of $\ell_1$ noise would be $\cO(n/\epsilon)$. To improve on this, we first smooth the input prevalence vector such that it is non-zero only for few values of $r$ and show that the smoothing method also reduces the sensitivity of cumulative prevalences and hence reduces the amount of noise added. 
 
While applying smoothing is the core idea, two questions remain: how to select the location of non-zero values
and how to smooth to reduce the sensitivity? We now describe these technical details.
 
 \noindent \textbf{(H1) Approximate high prevalences}: Recall that $N$ was obtained by adding geometric noise to $n$. In the rare case that this geometric noise is very negative, then there can be prevalences much larger than $2N$. This can affect the smoothing step. To overcome this, we move all counts above $2N$ to $N$. Since this changes the histogram with low probability, it does not affect the $\ell_1$ error.
 
\noindent \textbf{(H2) Compute boundaries}: We find a set of boundaries $S$ and smooth counts to elements in $S$. Ideally we would like to ensure that there is a boundary close to every count. 
For small values of $r$, we ensure this by adding all the counts and hence there is no smoothing. If $r \approx \sqrt{n}$, we use boundaries that are uniform in the log-count space. However, using this technique for all values of $r$, results in an additional $\log n$ factor. To overcome this, for $r \gg \sqrt{n}$, we use the noisy large counts in step \textbf{(4)} to find the boundaries and ensure that there is a boundary close to every count. 

\noindent \textbf{(H3) Smooth prevalences}: The main ingredient in proving better utility in the high privacy regime is the smoothing technique, which we describe now with an example. Suppose that all histograms have non-zero prevalences only between counts $s$ and $s+t$ and further suppose we have two neighboring histograms $h^1$ and $h^2$ as follows. 
$\varphi^1_{r} = 1$ and for all $i \in [s, s+t] \setminus \{r\}$, $\varphi^1_i = 0$. Similarly, let $\varphi^2_{r+1} = 1$ and for all $i \in [s, s+t] \setminus \{r+1\}$, $\varphi^2_i = 0$. If we want to release the prevalences or cumulative prevalences, we add $\ell_1$ noise of $\cO(1/\epsilon)$ for each prevalence in $[s, s+t]$. Thus the $\ell_1$ norm of the noise would be $\cO(t/\epsilon)$. We propose to reduce this noise by smoothing prevalences.

For a $r \in [s, s+t]$ , we divide $\varphi_r$ into $\varphi_{s}$ and $\varphi_{s+t}$ as follows. We assign $\frac{s + t - r}{t}$ fraction of $\varphi_r$ to $\varphi_{s}$ and the remaining to $\varphi_{s+t}$. After this transformation, the first histogram becomes $h^{t1}$ given by, $\varphi^{t1}_{s} = \frac{t+s - r}{t}$ and $\varphi^{t1}_{s} = \frac{r}{t}$ and all other prevalences are zero. Similarly, the second histogram becomes $h^{t2}$ given by, $\varphi^{t2}_{s} = \frac{t+s - r-1}{t}$ and $\varphi^{t1}_{s} = \frac{r+1}{t}$ and all other prevalences are zero. Thus the prevalences after smoothing differ only in two locations $s$ and $s+t$ and they differ by at most $1/t$. Thus the total amount of noise needed for a DP release is $\cO(1/t\epsilon)$ to these two prevalences. However, note that we also incur a loss as due to smoothing, which can be shown to be $\cO(t)$. Hence, the total amount of error would be $\cO(1/(t\epsilon) + t)$. Choosing $t = 1/\sqrt{\epsilon}$, yields a total error of $\cO(1/\sqrt{\epsilon})$. The above analysis is for a toy example and extending it to general histograms requires additional work. In particular, we need to find the smoothing boundaries that give the best utility. As described above, we choose boundaries based on logarithmic partition of counts and also by private values of counts. The utility analysis with these choice of boundaries is in Appendix~\ref{app:high_utility}.

{\noindent}\textbf{(H4) Add small noise}: Since the prevalences are smoothed, we add small amount of noise to the corresponding cumulative prevalences. For $\varphi_{s_i+}$, we add $L(1/(s_{i} - s_{i-1})\epsilon)$ to obtain $\epsilon$-DP.

\noindent \textbf{(H5) Post-processing}: Finally, we post-process the prevalences similar to \textbf{(L1)} to impose monotonicity and ensure that the resulting prevalences are positive and non-negative integers.  

Since we used succinct histogram representation, ensured that the size of $S$ is small, and used counts larger than $\tcO(\sqrt{n\epsilon})$ to find boundaries, the expected run time is $\tcO\bigl(\sqrt{\frac{n}{\epsilon}}+ \frac{1}{\epsilon}\bigr)$ for $\epsilon \leq 1$. 

\textbf{Privacy budget allocation:} 
We allocate $\epsilon_1$ privacy budget to estimate $n$, $\epsilon_2$ to the rest of \textsc{PrivHist} and $\epsilon_3$ to \textsc{PrivHist-HighPrivacy}. We set $\epsilon_1 = \epsilon_2 =\epsilon_3$ in our algorithms. We note that there is no particular reason for $\epsilon_1$, $\epsilon_2$, and $\epsilon_3$ to be equal and we chose those values for simplicity and easy readability. 
For example, since $\epsilon_1$ is just used to estimate $n$, the analysis of the algorithm shows that $\epsilon_2, \epsilon_3$ affects utility more than $\epsilon_1$. Hence, we can set $\epsilon_2 = \epsilon_3 = \epsilon(1-o(1))/2$ and $\epsilon_1 = o(\epsilon)$ to get better practical results. Furthermore, for low privacy regime, the algorithm only uses a privacy budget of $\epsilon_1 + \epsilon_2$. Hence, we can set $\epsilon_1 = o(\epsilon)$, $\epsilon_2 = \epsilon(1-o(1))$, and $\epsilon_3 = 0$.

%% file: acknowledgements.tex
\section{Acknowledgments}
Authors thank Jayadev Acharya and Alex Kulesza for helpful comments and discussions. 

%% file: algorithm.tex
\begin{center}
\fbox{\begin{minipage}{1.0\textwidth}
\begin{center}
Algorithm \textsc{PrivHist}
\end{center}
\noindent\textbf{Input}: anonymized histogram $h$ in terms of prevalences i.e., $\{(r,\varphi_r) : \varphi_r > 0\}$, privacy cost $\epsilon$. \newline
\noindent\textbf{Parameters}: $\epsilon_1 = \epsilon_2  = \epsilon_3 = \epsilon/3$. \newline
\noindent\textbf{Output}: DP anonymized histogram $H$ and $N$ (an estimate of $n$).
\begin{enumerate}[leftmargin=0.5cm]
\item DP value of the total sum: $N = \max(\sum_{n_x \in h} n_x +Z^a, 0)$, where $Z^a \sim G(e^{-\epsilon_1})$. If $N=0$, output empty histogram and $N$. Otherwise continue.
\item Split $h$: Let $T = \lceil \sqrt{N \min(\epsilon, 1)} \rceil$ and $M = \left\lceil \frac{\max(2\log N e^{\epsilon_2}, 1)}{\epsilon_2} \right\rceil $.
\begin{enumerate}[leftmargin=0.3cm]
\item $H^a: \varphi^a_{T} = \varphi_{T} +M, \varphi^a_{T+1} = \varphi_{T+1} +M$ and $\forall r \notin \{T, T+1\}$, $ \varphi^a_r = \varphi_r$.
\item $H^b: \varphi^b_{T+1} = \varphi^a_{T+1} + Z^b, \varphi^b_T = \varphi^a_{T} - Z^b$ and $ \forall r \notin \{T, T+1\}$ $\varphi^b_r = \varphi^a_r$, where $Z^b \sim G(e^{-\epsilon_2})$.
\item Divide $H^b$ into two histograms $H^{bs}$ and $H^{b\ell}$. For all $ r \geq T+1$,
  $\varphi^{b\ell}_{r} = \max(0, \sum^r_{s=T+1} \varphi^b_r -  \sum^{r-1}_{s=T+1} \varphi^{b\ell}_{r})$ for all $r \leq T$ $\varphi^{bs}_{r} = \max(0, \sum^{T}_{s=r} \varphi^b_r - \sum^{T}_{s=r+1} \varphi^{bs}_{r})$.\arxiv{\footnote{We use the notation $\sum^{b}_{i=a} X = 0$, if $a > b$.}}
\end{enumerate}
\item DP value of $H^{bs}$. Let $Z^{cs}_r \sim G(e^{-\epsilon_2})$ i.i.d. and $H^{cs}$ be $\varphi^{cs}_{r^+} = \varphi^{bs}_{r^+} + Z^{cs}_r  \text {
  for } r \leq T$.
\item DP value of $H^{b\ell}$: Let $Z^{c\ell}_i \sim G(e^{-\epsilon_2})$ i.i.d. and $H^{c\ell}$ be $
N^{c\ell}_i = N^{b\ell}_{(i)} + Z^{c\ell}_i \text{ for } N_{(i)} \in H^{b\ell}$.
\item If $\epsilon > 1$, output \textsc{PrivHist-LowPrivacy}$(H^{cs}, H^{c\ell}, T, M)$ and $ N$.
\item If $\epsilon \leq 1$, output \textsc{PrivHist-HighPrivacy}$(h, H^{c\ell}, T, N, \epsilon_3)$ and $N$.
\end{enumerate}
\vspace{3ex}
\begin{center}
Algorithm \textsc{PrivHist-LowPrivacy}
\end{center}
\noindent\textbf{Input}: low-count histogram $H^{cs}$, high-count histogram $H^{c\ell}, T, M$ and
\noindent\textbf{Output}: a histogram $H$.
\begin{enumerate}[leftmargin=0.7cm]
\item[L1.] Post processing of $H^{cs}$:
\begin{enumerate}[leftmargin=0.3cm]
\item Find $\bar\varphi^\text{mon}$ that minimizes
$
\sum_{r \geq 1} (\varphi^\text{mon}_{r+} - \varphi_{r+}(H^{cs}))^2.
$ with $\varphi^\text{mon}_{r+} \geq \varphi^\text{mon}_{(r+1)+}, \forall r$. 
\item $H^{ds}$: for all $r$, $\varphi^{ds}_{r^+} = \text{round}(\max(\varphi^{\text{mon}}_{r^+}, 0))$.
\end{enumerate}
\item[L2.] Post processing of $H^{c\ell}$:  Compute $H^{d\ell} = \{ \max(N_i(H^{c\ell}), T), \forall i\}$.
\item[L3.] Let $H^d = H^{ds} + H^{d\ell}$.
\item[L4.] Compute $H^e$ by removing $M$ elements closest to $T+1$ from $H^d$ and then removing $M$ elements closest to $T$ and output it.
\end{enumerate}
\vspace{3ex}
\begin{center}
Algorithm \textsc{PrivHist-HighPrivacy}
\end{center}
\noindent\textbf{Input}: non-private histogram $h$, high-count histogram $H^{\ell}, T, N, \epsilon_3$ and
\noindent\textbf{Output}: a histogram $H$.
\begin{enumerate}[leftmargin=0.7cm]
\item[H1.] Approximate higher prevalences: for $r <2 N$, $\varphi^u_r = \varphi_r(h)$ and $\varphi^u_{2N} = \varphi_{2N+}(h)$.
\item[H2.] Compute boundaries: Let the set $S$ be defined as follows:
\begin{enumerate}[leftmargin=0.3cm]
\item $T' = \lceil 10 \sqrt{N/\epsilon^3_3} \rceil $, $q = \sqrt{\log(1/\epsilon_3) / N\epsilon_3}$
    \item $S = \{1, 2, \ldots, T\} \cup \{ \lfloor T (1+q)^i \rfloor: i \leq \log_{1+q} (T'/T) \} \cup \{ N_x : N_x \in H^{\ell}, N_x \geq T'\} \cup \{2N\} $.
\end{enumerate}
\item[H3.] Smooth prevalences: Let $s_i$ denote the
  $i^{\text{th}}$ smallest element in $S$.
\begin{enumerate}[leftmargin=0.3cm]
    \item $\varphi^{v}_{s_i} = \sum^{s_{i+1}}_{j = s_i} \varphi^u_j
      \cdot \frac{s_{i+1} - j}{s_{i+1} - s_{i}} + \sum^{s_{i}-1}_{j =
        s_{i-1}} \varphi^u_j \cdot \frac{j- s_{i-1}}{s_{i} -
        s_{i-1}}$ and if $j \notin S$, $\varphi^v_j = 0$.
    \end{enumerate}
    \item[H4.] DP value of $H^v$: for each $s_i \in S$, let $\varphi^w_{s_i+} = \varphi^v_{s_i+} + Z^{s_i}$, where
         $Z^{s_i} \sim L\left(\frac{1}{\epsilon_3 (s_i - s_{i-1})} \right)$.
\item[H5.] Find $H^{x}$ that minimizes  $\sum_{s_i \in S} (\varphi^x_{s_i+} - \varphi^w_{s_i+})^2 (s_{i} - s_{i-1})^2$ such that $\varphi^x_{s_i+} \geq \varphi^x_{s_{i+1}+} \forall i$.
\item[H6.] Return $H^y$ given by, $\varphi^\text{y}_{r^+} = \text{round}( \max(\varphi^{\text{x}}_{r^+}, 0)) \, \, \forall r$.
\end{enumerate}
\end{minipage}}
\end{center} 

%% file: geometric_distributions.tex
\section{Geometric mechanism}
\label{app:moments}
The mostly popular mechanism for $\epsilon$-DP is the Laplace mechanism, which is defined as follows. 
\begin{definition}[Laplace mechanism ($L(b)$)~\cite{dwork2014algorithmic}]
When the true query result is $f$, the mechanism outputs $f + Z$ where Z is a random variable distributed as a Laplace distribution distribution: $\Pr(Z = z) = \frac{1}{2b} \exp\left(- \frac{|x|}{b} \right)$ for every $z \in \RR$. If output of $f$ has sensitivity $\Delta$, then to achieve $\epsilon$-DP add $Z \sim L(\Delta/\epsilon)$.
\end{definition}

Since, we have integer inputs, we use the geometric mechanism:  
\begin{definition}[Geometric mechanism ($G(\alpha)$)~\cite{ghosh2012universally}]
\label{def:geometric}
 When the true query result is $f$, the mechanism outputs $f + Z$ where Z is a random variable distributed as a two-sided geometric distribution: $\Pr(Z = z) = \frac{1-\alpha}{1+\alpha} \cdot \alpha^{|z|}$ for every integer $z$. If output of $f$ is integers and has sensitivity $\Delta$ (an integer), then to achieve $\epsilon$-DP add $Z \sim G(e^{\epsilon/\Delta})$.
\end{definition}
\cite{ghosh2012universally} showed that geometric mechanism is universally optimal for a general class of functions under a Bayesian framework. Geometric mechanism is beneficial over Laplace mechanism in two ways:
the output space of the mechanism is discrete. Since we have integer inputs, this removes the necessity of adding rounding off the outputs. For $\epsilon$-DP, the expected $\ell_1$ noise added by the Laplace mechanism is $1/\epsilon$, which strictly larger than that of the geometric mechanism ($2e^{-\epsilon}/(1-e^{-2\epsilon})$) (see below). For moderate values of $\epsilon$, this difference is a constant. We now state few properties of the geometric distribution which are used in the rest of the paper.

We find the following set of equations useful in the rest of the paper. In the following let $Z_G \sim G(e^{-\epsilon})$ be a geometric random variable and $Z_L \sim L(1/\epsilon)$ be a Laplace random variable.
\begin{align*}
\EE[Z_G] &= 0 = \EE[Z_L]. \\
\EE[|Z_G|] &= \frac{2e^{-\epsilon}}{1-e^{-2\epsilon}} \leq \frac{1}{\epsilon} = \EE[|Z_L|]. \\
\EE[Z^2_G] & = \frac{2e^{-\epsilon}}{(1-e^{-\epsilon})^2} \leq \frac{2}{\epsilon^2} = \EE[Z^2_L]. 
\end{align*}
The next lemma bounds moments of $\max(n + Z, 0)$ when $Z$ is a zero mean random variable.
\begin{lemma}
\label{lem:geomoments}
Let $Z$ be a random variable and $n \geq 0$. If $Y = \max(n + Z, 0)$, then
\begin{equation*}
\EE[|Y -n|]\leq \EE|Z|,
\end{equation*}
and
\begin{equation*}
\EE \left[ \frac{\indic_{Y > 0} }{Y} \right] \leq \frac{1}{n} + \frac{\EE[Z^2]}{2n^2}.
\end{equation*}
\end{lemma}
\begin{proof}
To prove the first inequality, observe that 
\[
|Y - n| = |\max(Z, -n)| \leq |Z|. 
\]
Taking expectation yields the first equation. For the second term,
\begin{equation}
\label{eq:geomoments1}
\frac{1}{Y} = \frac{1}{n} + \frac{n - Y}{Yn} = \frac{1}{n} +
\frac{n-Y}{n^2} + \frac{(n-Y)^2}{n^2Y} \leq \frac{1}{n} +
\frac{n-Y}{n^2} + \frac{(n-Y)^2}{2n^2}.
\end{equation}
Furthermore, 
\[
(n-Y)\indic_{Y > 0} = - Z \indic_{Y > 0} = - Z \indic_{-Z < n} \leq  - Z. 
\]
Combining the above two equations and using the fact that $|Y-n| \leq |Z|$ 
yields the second equation in the lemma.
\end{proof}

%% file: distance_metric.tex
\section{Properties of the distance metric}
\label{app:distance}
\begin{proof}[Proof of Lemma~\ref{lem:distance}]
Recall that the distance between to histograms is given by
\begin{align*}
\ell_1(h_1, h_2)
& \eqdef \min_{D_1, D_2 : h(D_1) = h_1, h(D_2) = h_2} \ell_1(D_1, D_2) \\
&= \min_{D_1, D_2 : h(D_1) = h_1, h(D_2) = h_2} 
\sum_{x \in \cX} |n_x(D_1) - n_x(D_2)|.
\end{align*}
Let $D^*_1$ and $D^*_2$ be the datasets that achieve the minimum above. Consider any two labels $x,y$ such that $n_x(D^*_1) \geq n_y(D^*_1)$. Let $D'_2$ be the dataset obtained as follows: $n_y(D'_2) = n_x(D^*_2)$ and $n_x(D'_2) = n_y(D^*_2)$ and for all other $z \notin \{x,y\}$, $n_z(D'_2) = n_z(D^*_2)$. Since $D^*_2$ is the optimum,
\begin{align*}
\ell_1(D^*_1, D^*_2) \leq \ell_1(D^*_1, D'_2).
\end{align*}
Expanding both sides and canceling common terms, we get,
\begin{align*}
|n_x(D^*_1) - n_x(D^*_2)| + |n_y(D^*_1) - n_y(D^*_2)|
& \leq |n_x(D^*_1) - n_x(D'_2)| + |n_y(D^*_1) - n_y(D'_2)| \\
& \leq |n_x(D^*_1) - n_y(D^*_2)| + |n_y(D^*_1) - n_x(D^*_2)|,
\end{align*}
and thus if $n_x(D^*_1) \geq n_y(D^*_1)$, then $n_x(D^*_2) \geq n_y(D^*_2)$. Hence, the label of the $i^{th}$ highest count in both the datasets should be the same and
\begin{align*}
\ell_1(D^*_1, D^*_2) = \sum_{x \in \cX} |n_x(D^*_1) - n_x(D^*_2)| = \sum_{i \geq 1} |n_{(i)}(h_1) - n_{(i)}(h_2)|.
\end{align*}
\end{proof}
The distance measure satisfies triangle inequality, i.e., for any three histograms $h_1, h_2$, and $h_3$,
\[
\ell_1(h_1, h_2) \leq \ell_1(h_1, h_3) + \ell_1(h_2, h_3).
\]
The proof of the above equation is a simple consequence of Lemma~\ref{lem:distance} and is omitted. We now show that dividing histograms only increases the distance.
\begin{lemma}
\label{lem:dissociation}
If $h = h_1 + h_2$ and $h' = h'_1 + h'_2$, then
\begin{equation*}
\ell_1(h, h') \leq \ell_1(h_1, h'_1) + \ell_1(h_2, h'_2).
\end{equation*}
\end{lemma}
\begin{proof}
Since the elements in $h_1+h_2$ are same as elements in $h$ and elements in $h'_1+ h'_2$ are same as elements in $h_2$, there exists a permutation $\sigma$ such that 
\begin{align*}
\ell_1(h_1, h'_1) + \ell_1(h_2, h'_2)
& = \sum_{i \geq 1} | n_{(i)}(h_1) - n_{(i)}(h'_1)| + 
\sum_{i \geq 1} | n_{(i)}(h_2) - n_{(i)}(h'_2)| \\
& = \sum_{i \geq 1} |n_{(i)}(h) - n_{(\sigma_{i})}(h')|.
\end{align*}
Similar to proof of Lemma~\ref{lem:distance}, it can be shown that the $\sigma$ that minimizes the above sum is the one that matches $i^{\text{th}}$ highest count in $h$ to $i^{\text{th}}$ highest count in $h'$ and hence
\begin{align*}
\ell_1(h, h') = \sum_{i \geq 1} |n_{(i)}(h) - n_{(i)}(h')|
\leq \sum_{i \geq 1} |n_{(i)}(h) - n_{(\sigma_{i})}|.
\end{align*}
\end{proof}
It is useful to have few upper bounds on the $\ell_1$ distance over histograms.
\begin{lemma}
\label{lem:diff_distance}
For any two histograms $h_1, h_2$, 
\begin{align}
\ell_1(h_1, h_2) 
  \leq \sum^{r_{\max}(h_1,h_2)}_{r \geq 1} |\varphi_{r^+}(h_1) - \varphi_{r^+}(h_2)|
\leq \sum_{r\geq 1} r |\varphi_r(h_1) - \varphi_r(h_2)|
, \label{eq:dd100}
\end{align}
where $r_{\max}(h_1,h_2)$ is the maximum $r$ such that $\varphi_r(h_1) + \varphi_r(h_2) > 0$. 
\end{lemma}
\begin{proof}
We prove the first inequality by induction on $r_{\max}(h_1, h_2)$. Suppose $r_{\max}(h_1, h_2) = 1$, then the inequality holds trivially as 
\begin{align*}
\sum_{i} |n_{(i)}(h_1) - n_{(i)}(h_2)| = |\varphi_1(h_1) - \varphi_1(h_2)| = \sum^{r_{\max}(h_1, h_2)}_{r=1} |\varphi_{r^+}(h_1) - \varphi_{r^+}(h_2)|.
\end{align*}
Now suppose it holds for all $r_{\max}(h_1, h_2) < r_0$. For $r_0 \eqdef r_{\max}(h_1, h_2)$. Let $h'_1$ and $h'_2$ be two datasets obtained as follows:
\[
h'_i = \max(n_x, r_0-1): n_x \in h_i\}.
\]
This mapping preserves the ordering of $n_{(i)}$s up to ties and $r_{\max}(h'_1, h'_2) = r_0 - 1$. Thus, 
\begin{equation}
\label{eq:induction}
\ell_1(h'_1, h'_2) \leq \sum^{r_{\max}(h'_1, h'_2)}_{r=1} |\varphi_{r^+}(h'_1) - \varphi_{r^+}(h'_2)| = \sum^{r_{\max}(h_1, h_2) - 1}_{r=1} |\varphi_{r^+}(h_1) - \varphi_{r^+}(h_2)|.
\end{equation}
Hence,
\begin{align*}
& \sum_{i} |n_{(i)}(h_1) - n_{(i)}(h_2)| \\ 
& = \sum_{i : \max(n_{(i)}(h_1), n_{(i)}(h_2)) < r_0} |n_{(i)}(h_1) - n_{(i)}(h_2)|  + \sum_{i : \max(n_{(i)}(h_1), n_{(i)}(h_2)) \geq r_0} |n_{(i)}(h_1) - n_{(i)}(h_2)| \\
& = \sum_{i : \max(n_{(i)}(h_1), n_{(i)}(h_2)) < r_0} |n_{(i)}(h'_1) - n_{(i)}(h'_2)| \\
& + \sum_{i : \max(n_{(i)}(h_1), n_{(i)}(h_2)) \geq r_0} |n_{(i)}(h'_1) - n_{(i)}(h'_2) + \indic_{n_{(i)}(h_1)=  r_0} - \indic_{n_{(i)}(h'_2)= r_0}| \\
& \leq \sum_{i : \max(n_{(i)}(h_1), n_{(i)}(h_2)) < r_0} |n_{(i)}(h'_1) - n_{(i)}(h'_2)| + \sum_{i : \max(n_{(i)}(h_1), n_{(i)}(h_2)) \geq r_0} |n_{(i)}(h'_1) - n_{(i)}(h'_2)| \\
& + \sum_{i : \max(n_{(i)}(h_1), n_{(i)}(h_2)) \geq r_0} | \indic_{n_{(i)}(h_1)= r_0} - \indic_{n_{(i)}(h'_2)=r_0}| \\
& = \ell_1(h'_1, h'_2) + |\varphi_{r_0}(h_1) - \varphi_{r_0}(h_2) |
\end{align*}
Combining the above equation with Equation~\eqref{eq:induction} yields the first inequality.
For the second inequality, observe that
\begin{align*}
\sum^{r_{\max}(h_1,h_2)}_{r \geq 1} \left \lvert\varphi_{r^+}(h_1) - \varphi_{r^+}(h_2) \right \rvert
& = \sum^{r_{\max}(h_1,h_2)}_{r \geq 1} |\sum_{s \geq r}\varphi_{s}(h_1) - \varphi_{s}(h_2)| \\
& \leq \sum^{r_{\max}(h_1,h_2)}_{r \geq 1} \sum_{s \geq r} |\varphi_{s}(h_1) - \varphi_{s}(h_2)| \\
& = \sum_{s \geq 1} s |\varphi_{s}(h_1) - \varphi_{s}(h_2)|,
\end{align*}
where the inequality follows by triangle inequality and the last equality follows by observing that each term corresponding to index $s$ appears exactly $s$ times.
\end{proof}
We now show a simple property of rounding off integers.
\begin{lemma}
\label{lem:round}
Let $x_1,x_2,\ldots, x_n$ be integers. Let $y_1,y_2,\ldots, y_n$ be real numbers. Let $\hat{y}_i$ be the nearest integer to $y_i$. Then,
\[
\sum^n_{i=1} | x_i - \hat{y}_i| \leq 2\sum^n_{i=1} | x_i - {y}_i|.
\]
\end{lemma}
\begin{proof}
For any $i$,
\[
| x_i - \hat{y}_i|  \leq | x_i - {y}_i|  + | y_i - \hat{y}_i| \leq 2| x_i - {y}_i|,
\]
where the second inequality follows from the observation that $\hat{y}_i$ is the nearest integer to $y_i$. Summing over all indices $i$ yields the lemma.
\end{proof}
We need the next auxilllary lemma, which we use in the proofs.
\begin{lemma}
\label{lem:addm}
For a histogram $h_1$, let $h'_1$ be the histogram obtained by adding $k$ elements of value $t$ to $h$. Let $h'_2$ be another histogram and let $h_2$ is obtained by removing $k$ elements that are closest to $t$. Then
\[
\ell_1(h_1, h_2) \leq 2\ell_1(h'_1, h'_2).
\]
\end{lemma}
\begin{proof}
Let $h''_2$ be the histogram obtained by adding $k$ elements of value $t$ to $h_2$. Since adding same number of elements to two datasets do not decrease the $\ell_1$ distance,
\begin{align*}
  \ell_1(h_1, h_2)= \ell_1(h'_1, h''_2) \leq \ell_1(h'_1, h'_2) + \ell_1(h'_2, h''_2),
\end{align*}
where the second inequality follows by triangle inequality. Consider the set of all histograms that have $\varphi_{t} \geq k$. Both $h''_2$ and $h'_1$ belong to this set. 
It can be shown that of all histograms in that set $h''_2$ is closest to $h'_2$ and hence
\[
 \ell_1(h'_2, h''_2) \leq \ell_1(h'_2, h'_1),
\]
and hence the lemma.
\end{proof}

%% file: privacy_analysis.tex
\section{Privacy analysis of \textsc{PrivHist}}
\label{app:dpanalysis}

\subsection{Overview of privacy analysis}
\label{app:privacy_overview}

We break the analysis of \textsc{PrivHist} step by step. We will show
that release of $N$ (step $(1)$) is $\epsilon_1$-DP. Then, we show that $H^{cs}$ and $H^{c\ell}$ are $\epsilon_2$-DP. Observe that \textsc{PrivHist-LowPrivacy} is
just a post-processing step and by the post processing lemma does not
need any differentialy privacy analysis. Finally we show that
\textsc{PrivHist-HighPrivacy} is $\epsilon_3$ differentially
private. By the composition theorem~\cite{dwork2014algorithmic}, it
follows that the total privacy cost is $\epsilon_1 + \epsilon_2 +
\epsilon_3= \epsilon$ and hence the privacy cost in
Theorem~\ref{thm:main}. Of the above steps, proving $N$ is $\epsilon_1$-DP is straightforward and a sketch is in Lemma~\ref{lem:N}.
Proving $H^{cs}$ and $H^{c\ell}$ is $\epsilon_2$-DP is more involved and is in Lemma~\ref{lem:c_privacy}. The main intuition behind Lemma~\ref{lem:c_privacy} is a sensitivity preserving histogram split, which we describe below.

\subsection{Sensitivity preserving histogram split}
\label{app:split}
Any two neighboring datasets $h_1$ and $h_2$ can fall into one of three categories:
\begin{enumerate}
\item They differ in $\varphi_{T}$ and $\varphi_{T+1}$.
\item They differ in $\varphi_r$ and $\varphi_{r+1}$ for some $0 \leq r < T - 1$.
\item They differ in $\varphi_r$ and $\varphi_{r+1}$ for some $r > T$.
\end{enumerate}
For cases $2$ and $3$ above, it suffices to add noise to cumulative prevalences and counts as in \textbf{(3)} and \textbf{(4)}. However, if they differ in $\varphi_{T}$ and $\varphi_{T+1}$, the analysis is more involved. For example, consider the following simple example. $h_1 = \{\varphi_{T} = 1, \varphi_{n -T} = 1\}$ and $h_2 = \{\varphi_{T+1} = 1, \varphi_{n -T-1} = 1\}$. $h_1$ and $h_2$ have $\ell_1$ distance of one and are neighbors. If we divide $h_1$ in to two parts based on threshold $T$, say $h^s_1 = \{\varphi_{T} = 1\}$ and $h^\ell_1 = \{\varphi_{n -T} = 1\}$ and $h^s_2 = \{\}$ and $h^\ell_2 =  \{\varphi_{T+1} = 1, \varphi_{n -T-1} = 1\}$, then $\ell_1(h^\ell_1, h^\ell_2) = T + 2$. Thus, if we naively add noise to cumulative prevalences for $r \leq T$ and to counts $r > T$, then we need to add noise $L(\cO(T/\epsilon))$, which makes the utility of the algorithm much worse. To overcome this, we preprocess $h$ by moving $Z^b$ counts from $\varphi_T$ to $\varphi_{T+1}$, where $Z_b$ is a geometric random variable. This provides the required privacy without increasing the utility considerably. Finally, moving mass $Z^b$ can make the histogram to have negative prevalences. To overcome this, we add $M$ fake counts to $\varphi_T$ and $\varphi_{T+1}$.

\subsection{Technical details}

We first prove a dataset depending composition theorem that helps us decompose differential privacy 
analysis depending on the dataset.
\begin{theorem}[Dataset dependent composition theorem]
\label{lem:composition}
Let $Z_1, Z_2,\ldots, Z_n $ be a set of independent random variables. Let $X_1 = f_1(x, Z_1)$ be a deterministic function. Similarly let $X_{i} = f_i(X_{i-1}, Z_i)$ be deterministic functions for $2 \leq i \leq n$. If for any two neighboring data sets $x$ and $x'$, 
\begin{equation}
    \label{eq:composition2}
\min_{i \geq 1} \max_{z_1, z_2, \ldots, z_{i-1}, x_i} \frac{\Pr(X_i = x_i | x, z_1, z_2, \ldots, z_{i-1})}{\Pr(X_i = x_i | x', z_1, z_2, \ldots, z_{i-1})} \leq e^{\epsilon}, \footnote{For notational simplicity, let $z^0_1 = \emptyset$.}
\end{equation}
then $X_n$ is an $\epsilon$-DP output.
\end{theorem}
\begin{proof}
For any two datasets, since $x \to X_1 \to X_n$ is a Markov chain,
\[
\Pr(X_n = x_n | x)  = \sum_{x_i} \Pr(X_n = x_n | X_i = x_i) \cdot \Pr(X_i = x_i | x).
\]
Hence for any two datasets $x, x'$ any $x_n$, and for all $i$,
\begin{align*}
    \frac{\Pr(X_n = x_n | x)}{\Pr(X_n = x_n | x')}
    = \frac{\sum_{x_1} \Pr(X_n = x_n | X_i = x_i) \cdot \Pr(X_i = x_i | x)}{\sum_{x_i} \Pr(X_n = x_n | X_i = x_i) \cdot \Pr(X_i = x_i | x')} \leq \max_{x_i} \frac{ \Pr(X_i = x_i | x')}{\Pr(X_i = x_i | x)}.
\end{align*}
Similarly for any $i$,
\[
\Pr(X_i = x_i | x)  = \sum_{z_1, z_2,\ldots, z_{i-1}} \Pr(X_i= x_i | x, z_1, z_2, \ldots, z_{i-1}) \cdot \Pr(z_1, z_2, \ldots, z_{i-1}).
\]
Hence for any two datasets $x, x'$,
\begin{align*}
    \frac{\Pr(X_i = x_i | x) }{\Pr(X_i = x_i | x') }
& = \frac{\sum_{z_1, z_2,\ldots, z_{i-1}} \Pr(X_i= x_i | x, z_1, z_2, \ldots, z_{i-1}) \cdot \Pr(z_1, z_2, \ldots, z_{i-1})}{\sum_{z_1, z_2,\ldots, z_{i-1}} \Pr(X_i= x_i | x', z_1, z_2, \ldots, z_{i-1}) \cdot \Pr(z_1, z_2, \ldots, z_{i-1})} \\
& \leq 
\max_{z_1, z_2, \ldots, z_{i-1}} \frac{\Pr(X_i= x_i | x, z_1, z_2, \ldots, z_{i-1})}{ \Pr(X_i= x_i | x', z_1, z_2, \ldots, z_{i-1})}.
\end{align*}
Hence,
\begin{align*}
\max_{x_n} \frac{\Pr(X_n = x_n | x)}{\Pr(X_n = x_n | x')}
&\leq \min_{i \geq 1} \max_{x_i} \frac{ \Pr(X_i = x_i | x')}{\Pr(X_i = x_i | x)} \\
& \leq \min_{i \geq 1}  \max_{x_i} \max_{z_1, z_2, \ldots, z_{i-1}} \frac{\Pr(X_i= x_i | x, z_1, z_2, \ldots, z_{i-1})}{ \Pr(X_i= x_i | x', z_1, z_2, \ldots, z_{i-1})},
\end{align*}
and hence for every pair of datasets if the right hand side is smaller than $e^{\epsilon}$, then $X_n$ is diffferentially private.
\end{proof}
\subsection{Privacy analysis}
We start with proving $N$ is $\epsilon_1$-DP.
\begin{lemma}
\label{lem:N}
$N$ is $\epsilon_1$-DP.
\end{lemma}
\begin{proof}[Proof sketch]
 The proof follows from Definition~\ref{def:geometric} and the fact
 that for any two neighboring datasets $h_1, h_2$, $n(h_1) - n(h_2) = | \sum_{n_x \in
   h_1} n_x - \sum_{n_x \in h_2} n_x| \leq \ell_1(h_1, h_2) = 1$, and
 hence the sensitivity of this query is $1$.
\end{proof}
We now show that the release of $H^{cs}$ and $H^{c\ell}$ is DP.
\begin{lemma}
\label{lem:c_privacy}
Release of $H^{cs}$ and $H^{c\ell}$ is $\epsilon_2$-DP.
\end{lemma}
\begin{proof}
$h \to H^{b} \to (H^{bs}, H^{b\ell}) \to (H^{cs}, H^{b\ell}) \to (H^{cs}, H^{c\ell})$ is a Markov chain. We use Theorem~\ref{lem:composition} to show that the output of this Markov chain is DP for all datasets.
For any two neighboring datasets $h_1$ and $h_2$ can fall into one of three categories:
\begin{enumerate}
\item They differ in $\varphi_{T}$ and $\varphi_{T+1}$.
\item They differ in $\varphi_r$ and $\varphi_{r+1}$ for some $0 \leq r < T - 1$.
\item They differ in $\varphi_r$ and $\varphi_{r+1}$ for some $r > T$.
\end{enumerate}
We prove that $ (H^{cs}, H^{c\ell})$ release is $\epsilon_2$-DP  for each of the above three cases.

\noindent \textbf{Case 1}: We show that the process $h \to H^{b}$
satisfies \eqref{eq:composition2}. Observe that
\begin{align*}
\max_{h^b} \frac{\Pr(H^b = h^b|h_2) }{\Pr(H^b = b^b|h_1) } 
& \stackrel{(a)}{=} 
\frac{\Pr(\bar{\varphi}(H^b) = \bar\varphi^b | \bar{\varphi}(h_2))}{\Pr(\bar{\varphi}(H^b) = \bar\varphi^b | \bar{\varphi}(h_1))} \\
& \stackrel{(b)}{=}
 \frac{\Pr({\varphi}_T(H^b) = \varphi^b_T , \varphi_{T+1}(H^b) = \varphi^b_{T+1} | {\varphi}_T(h_2))}{\Pr({\varphi}_T(H^b) = \varphi^b_T , \varphi_{T+1}(H^b) = \varphi^b_{T+1}  | \bar{\varphi}_T(h_1))} \\
& \stackrel{(c)}{=} \frac{\Pr(Z^b = \varphi^b_T - \varphi_T(h_2))}{\Pr(Z^b = \varphi^b_T - \varphi_T(h_1))} \\
& =
 \frac{e^{-\epsilon_2|\varphi^b_T - \varphi_T(h_2)|}}{e^{-\epsilon_2|\varphi^b_T - \varphi_T(h_1)|}} \\
& \leq e^{\epsilon_2 |\varphi_T(h_2) - \varphi_T(h_1)|} \\
& \leq e^{\epsilon_2}.
\end{align*}
$(a)$ follows by observing that there is a one to one correspondence
between the histogram and the prevalences, $(b)$ follows from the fact
that noise is added only to $\varphi_T$ and $\varphi_{T+1}$, and $(c)$
follows from the fact that the noise added to $\varphi_{T+1}$ is a deterministic function of noise added to
$\varphi_{T}$, and $\varphi^b_{T} + \varphi^b_{T+1} = \varphi^a_{T} +
\varphi^a_{T+1}$ always.

\noindent \textbf{Case 2}: We show that $(H^{bs}, H^{b\ell})
\to (H^{cs}, H^{b\ell})$ satisfies \eqref{eq:composition2}. Let
$H^{bs}_1$ and $H^{bs}_2$ be the values of $H^{bs}$ for inputs $h_1$
and $h_2$ respectively. For any $Z_b = z_b$, since difference of maximums is at most maximum of differences,
\begin{align*}
\sum_{r} |\varphi_{r^+}(H^{bs}_2) - \varphi_{r^+}(H^{bs}_1)| 
& \leq 
\sum_{r} |\varphi_{r^+}(h_2) - z_b- \varphi_{r^+}(h_1) + z_b| \\
& =
\sum_{r} |\varphi_{r^+}(h_2) - \varphi_{r^+}(h_1)| \leq 1.
\end{align*}
Hence,
\begin{align*}
\frac{\Pr((H^{cs}, H^{b\ell}) = (h^{cs}, h^{b\ell}) | H^{bs}_2,
  H^{b\ell}) }{\Pr(H^{cs}, H^{b\ell}) = (h^{cs}, h^{b\ell}) |
  H^{bs}_1, H^{b\ell})} & = \frac{\Pr(H^{cs} = h^{cs}| H^{bs}_2)
}{\Pr(H^{cs} = h^{cs} | H^{bs}_1)} \\ 
& = \frac { \Pr(
  \bar\varphi_{r^+}(H^{cs}) = \bar\varphi_{r^+}(h^{cs})| \bar\varphi_{r^+}(H^{bs}_2))}{ \Pr( \bar\varphi_{r^+}(H^{cs}) =
  \bar\varphi_{r^+}(h^{cs})| \bar\varphi_{r^+}(H^{bs}_1) )}
  \\ & =
\prod_{r\geq 1}\frac{\Pr(\varphi_{r^+}(H^{cs}) =
  \varphi_{r^+}(h^{cs})|\bar\varphi_{r^+}(H^{bs}_2))}{
  \Pr(\varphi_{r^+}(H^{cs}) = \varphi_{r^+}(h^{cs})|
  \bar\varphi_{r^+}(H^{bs}_1))} \\ & = \prod_{r\geq 1}
\frac{e^{-\epsilon_2|\varphi_{r^+}(h^{cs}) - \varphi_{r^+}(H^{bs}_2)|
}}{e^{-\epsilon_2|\varphi_{r^+}(h^{cs}) - \varphi_{r^+}(H^{bs}_1)| }}\\ &
\leq \prod_{r\geq 1} e^{\epsilon_2|\varphi_{r^+}(H^{bs}_2) -
  \varphi_{r^+}(H^{bs}_1)|} \leq e^{\epsilon_2}.
\end{align*}
The first equality follows from the observation that there is a one to
one correspondence between $\varphi_r$s and $\varphi_{r^+}$s. The
second equation follows from the fact that noise added to various
\ $\varphi_{r^+}$s are independent of each other. The rest of the
proof follows from the definition of the geometric mechanism and the fact
that $h_1$ and $h_2$ are neighbors.

\noindent{\textbf{Case 3}}: We show that $(H^{cs}, H^{b\ell}) \to
(H^{cs}, H^{c\ell})$ satisfies \eqref{eq:composition2}. Let
$H^{b\ell}_1$ and $H^{b\ell}_2$ are the $H^{b\ell}$'s corresponding to
$h_1$ and $h_2$ respectively.  Conditioned on the value of $Z^b$, for
any two neighboring histograms that differ in two consecutive $r$'s
that are larger than $T$,
\[
\sum_{r > T} \varphi_r(H^{b\ell}_2) - \sum_{r > T} \varphi_r(H^{b\ell}_2) =
\sum_{r > T} \varphi_r(h_2) - \varphi_r(h_1) = 0.
\]
Since their sums are equal, conditioned on the value of $Z^b$, it can
be shown that $H^{b\ell}_1, H^{b\ell}_2$ are proper histograms and
both contain $\max(0, M + Z^b + \sum_{r > T} \varphi_r(h_1))$ elements
and they differ in at most two consecutive values of $r$. Thus
$H^{b\ell}_1$ and $H^{b\ell}_2$ contain same number of counts and
differ in at most one count denoted by $i^*$.  With these observations
we now bound the ratio of probabilities for differential privacy.

Since there is a one to one correspondence between sorted counts and
the histograms, we get
\begin{align*}
\frac{\Pr((H^{cs}, H^{c\ell}) = (h^{cs}, h^{c\ell}) | H^{b\ell}_2,
  H^{cs}) }{\Pr(H^{cs}, H^{c\ell}) = (h^{cs}, h^{c\ell}) |
  H^{b\ell}_1, H^{cs})} & = \frac{\Pr(H^{c\ell} = h^{c\ell}| H^{b\ell}_2)
}{\Pr(H^{c\ell} = h^{c\ell} | H^{b\ell}_1)} \\ 
&= \frac{\prod_{i} \Pr(N_i = n_i | n_{(i)}(H^{b\ell}_2))}{\prod_{i}
  \Pr(N_i = n_i | n_{(i)}(H^{b\ell}_2))} \\
&= \frac{\Pr(N_{i^*} =
  n_{i^*}| n_{({i^*})}(H^{b\ell}_2))}{\Pr(N_{i^*} = n_{i^*}|
  n_{({i^*})}(H^{b\ell}_1))} \\ & = \frac{e^{-\epsilon_2|n_{i^*} -
    n_{({i^*})}(H^{b\ell}_2)|}}{e^{-\epsilon_2|n_{i^*} -
    n_{({i^*})}(H^{b\ell}_1)|}} \leq
e^{\epsilon_2|n_{({i^*})}(H^{b\ell}_2) - n_{({i^*})}(H^{b\ell}_1)|}
\leq e^{\epsilon_2},
\end{align*}
where the last set of inequalities follow from the definition of
geometric mechanism and the fact that the noise added to $N_{(i)}$s
are independent of each other, and $n_{(i^*)}$ is the only count in
which the two histograms differ.
\end{proof}

We now show that \textsc{PrivHist-HigPrivacy} is $\epsilon_3$-DP. 
\begin{lemma}
\label{lem:highprivacy_dp}
The output of \textsc{PrivHist-HigPrivacy} is $\epsilon_3$-DP. 
\end{lemma}
\begin{proof}
Observe that $h \to H^v \to H^w \to H^x \to H^y$ is a Markov
chain, hence it suffices to prove $H^v \to H^w$ is $\epsilon_3$-DP.

Let $h_1$ and $h_2$ are two neighboring datasets. Without loss of
generality, let they differ in $\varphi_k$ and $\varphi_{k+1}$. Let
$h^v_1$ and $h^v_2$ be the two histograms obtained after quanitzation
step $(3)$. Since $h_1$ and $h_2$ differ only at $k$ and $k+1$,
$h^v_1$ and $h^v_2$ differ only in $\varphi^v_{s_{i-1}}$ and
$\varphi^v_{s_{i}}$ for some $i$. Furthermore,
\[
| \varphi^v_{s_{i-1}}(h^v_1) - \varphi^v_{s_{i-1}}(h^v_2) | = | \varphi^v_{s_i}(h^v_1) - \varphi^v_{s_i}(h^v_2) | \leq \frac{1}{s_{i} - s_{i-1}}.
\]
For all $j \notin \{s_{i-1}, s_{i}\}$, $\varphi^v_j(h^v_1) =\varphi^v_j(h^v_2)$.  Hence
$\varphi^v_{s_{i}+}(h^v_1)  \neq \varphi^v_{s_{i}+}(h^v_2)$ for only one value of $i$, let $i^*$ be this value.
\begin{align*}
\frac{\Pr(H^w = h_w | h^v_1)}{\Pr(H^w = h_w | h^v_2)} 
& = \prod_{i}
\frac{\Pr(\varphi_{s_i+}(H^w) = \varphi_{s_i+}(h^w) |
  \varphi_{s_i+}(h^v_1))}{\Pr(\varphi_{s_i+}(H^w) =
  \varphi_{s_i+}(h^w) | \varphi_{s_i+}(h^v_2))} \\ 
& = \frac{\Pr(\varphi_{s_i^*+}(H^w) = \varphi_{s_i^*+}(h^w)
  | \varphi_{s_i^*+}(h^v_1))}{\Pr(\varphi_{s_i^*+}(H^w) =
  \varphi_{s_i^*+}(h^w) | \varphi_{s_i^*+}(h^v_2))} \\ 
& = 
\frac{\Pr(Z^{s_{i^*}} = \varphi_{s_i^*+}(h^w) - \varphi_{s_{i^*+}}(h^v_1))}
{\Pr(Z^{s_{i^*}} = \varphi_{s_i^*+}(h^w) - \varphi_{s_{i^*+}}(h^v_2))} \\
& \leq
\exp \left ( \frac{\epsilon_3 (s_{i^*} - s_{i^*-1})}{s_{i^*} - s_{i^*-1}} \right)
\leq e^{\epsilon_3}.
\end{align*}
\end{proof}

%% file: utility_analysis.tex
\section{Utility analysis of \textsc{PrivHist}}
\label{app:analysis}

Let $H^o$ be the output of either \textsc{PrivHist-LowPrivacy} or \textsc{PrivHist-HighPrivacy}. In both the low and high privacy regimes, the output error can be bounded as
\[
\ell_1(h, H^o) \indic_{N > 0} + n \indic_{N \leq 0}.
\]
Furthermore, 
\[
\EE[n \indic_{N \leq 0}] \leq n e^{-n\epsilon_1} \lesssim e^{-\epsilon_1/2},~\footnote{We use $\lesssim$ instead of $\cO$ notation and $\gtrsim$ instead of $\Omega$ notation for compactness.}
\]
where the last inequality, follows by breaking it in to cases $n = 0$ and $n > 0$. 
The bound on $\EE[|N-n|]$ follows from the fact that $N = n +
G(e^{-\epsilon_1})$, Lemma~\ref{lem:geomoments}, and the moments of the geometric distribution. In the next two sections, we bound $\ell_1(h, H^o) \indic_{N > 0}$ for both low privacy and high privacy regimes.

\subsection{Low privacy regime}

In the following analysis, let $Z$ be a geometric random variable distributed as
$G(e^{-\epsilon_2})$. Let $H^{b'} = H^{bs} + H^{b\ell}$.  
For the bounds on the histogram, by
Lemma~\ref{lem:addm} and triangle inequality,
\begin{equation}
\label{eq:low1}
\ell_1(h, H^e) \leq 4 \ell_1(H^a, H^d) \leq 4 \ell_1(H^{b'}, H^d) + 4 \ell_1(H^{b'}, H^a).
\end{equation}
For the second term, observe that since $H^{b'}$ is obtained by moving
$Z^b$ terms between $T$ and $T+1$ and then majorizing it. If $|Z^b| \leq M$, then majorization does not change the histogram. Hence,
\[
\ell_1(H^{b'}, H^a) = \ell_1(H^{b'}, H^a)  \indic_{|Z^b| \leq M} +  \ell_1(H^{b'}, H^a)  \indic_{|Z^b|  > M} \leq |Z^b| + (n + 2M(T+1))\indic_{|Z^b|  > M}.
\]
Taking expectation on both sides
\begin{equation}
\label{eq:low2}
\EE_N[\ell_1(H^{b'}, H^a) ] \leq \EE[|Z^b| + (n + 2M(T+1))\indic_{|Z^b|  > M}] \lesssim
 \EE[Z] + \frac{n + N}{N^2} e^{-2\epsilon_2}. \footnote{$\EE_N$ denotes conditional expectation w.r.t. $N$}
\end{equation}
For the first term, by Lemma~\ref{lem:dissociation},
\begin{equation}
\label{eq:low3}
\ell_1(H^{b'}, H^d) \leq \ell_1(H^{bs}, H^{ds}) + \ell_1(H^{b\ell}, H^{d\ell}).
\end{equation}
We now bound both the terms above. For the large counts, let $i_j$ be the index of the noisy count $N_{(i)}(H^{b\ell})$.
\begin{align*}
 \ell_1(H^{b\ell}, H^{d\ell}) 
 &\leq \sum_{i} | N_{(i)} (H^{b\ell}) - N_{(i)} (H^{d\ell})|  \\
 & \leq \sum_{i} | N_{(i)} (H^{b\ell}) - N_{i_j} (H^{d\ell})|   \\
 & \leq \sum_{i} | N_{(i)} (H^{b\ell}) - N_{i_j} (H^{c\ell})|.
\end{align*}
Since the number of terms above $T$ is at most $n/T + M + Z^b$, in expectation,
\begin{equation}
\label{eq:low4}
\EE_N[  \ell_1(H^{b\ell}, H^{d\ell})  ] \leq \EE\left[\sum_{i} |Z^{c\ell}_i|\right]
\lesssim \left(\frac{n}{T} + M + \EE|Z| \right) \cdot \EE[|Z|].
\end{equation}
For the smaller counts observe that
\begin{align}
\ell_1(H^{bs}, H^{ds}) 
&\leq \sum_{r \geq 0} |\varphi^{bs}_{r+}
- \varphi^{ds}_{r+} | \nonumber \\
& \stackrel{(a)}{\leq} 2 \sum_{r \geq 0} |\varphi^{bs}_{r+}
- \varphi^{\text{mon}}_{r+} | \nonumber \\
& \stackrel{(b)}{\leq} 2 \sqrt{T} \cdot \sqrt{ \sum_{r \geq 0} (\varphi^{bs}_{r+}
- \varphi^{\text{mon}}_{r+} )^2 } \nonumber \\
& \stackrel{(c)}{\leq} 2 \sqrt{T} \cdot \sqrt{ \sum_{r \geq 0} (\varphi^{bs}_{r+}
- \varphi^{\text{cs}}_{r+} )^2 } \nonumber \\
& \leq 2 \sqrt{T} \cdot \sqrt{ \sum_{r \geq 0} (Z^{cs}_r - Z^b)^2 }, \label{eq:low5}
\end{align}
where $(a)$ follows from the fact that rounding off increases the
error at most by $2$ (Lemma~\ref{lem:round}), $(b)$ follows by the Cauchy-Schwarz
inequality, and $(c)$ follows from the fact that $\varphi^{bs}_{r+}$
are monotonic and hence monotonic projection only decreases the error.
Hence in expectation,
\begin{equation}
    \label{eq:low6}
\EE_N[\ell_1(H^{bs}, H^{ds})] \lesssim T \sqrt{\EE[Z^2]}.
\end{equation}
Combining~\eqref{eq:low1},~\eqref{eq:low2},~\eqref{eq:low3},~\eqref{eq:low4},~\eqref{eq:low5}, and~\eqref{eq:low6},
\[
\EE_N[\ell_1(h, H^e)] \lesssim 
|Z| + 4 T \sqrt{\EE[Z^2]} + 2 \left( \frac{n}{T} + M + \EE|Z| \right) \cdot \EE[|Z|] + \frac{n+N}{N^2} e^{-2\epsilon_2}.
\]
Substituting $T = \lceil\sqrt{N} \rceil$ and $M = \lceil \frac{2\log
  Ne^{\epsilon_2}}{\epsilon_2} \rceil $, yields
\[
\EE_N[\ell_1(h, H^e)] \lesssim |Z| + \sqrt{N} \sqrt{\EE[Z^2]} + 2
\left( 1 + \frac{n}{\sqrt{N}} + \frac{\log N\epsilon_2}{\epsilon_2} + \EE|Z| \right)
\cdot \EE[|Z|] + \frac{n+N}{N} e^{-2\epsilon_2}.
\]
Taking expectation w.r.t. $N$ and using Lemma~\ref{lem:geomoments} yields
\[
\EE[\ell_1(h, H^e) \indic_{N > 0}] \lesssim  \sqrt{n \EE[Z^2]} + e^{-2\epsilon_2}  \lesssim  \sqrt{n} e^{-\epsilon/6},
\]
where the last inequality follows from moments of geometric distribution.
 \subsection{High privacy utility}
 \label{app:high_utility}
For the high privacy regime, by the triangle inequality,
\begin{equation}
    \label{eq:high1}
\ell_1(h, H^y) \leq \ell_1(h, H^u) +  \ell_1(H^u, H^y).
\end{equation}
For the first term, since we are only reducing counts of certain elements,
\begin{equation}
    \label{eq:high2}
\EE[\ell_{1}(h, H^u) ] \leq \EE[ \indic_{N < n/2} n] \leq
e^{-n\epsilon_1/2} n \lesssim \frac{1}{\epsilon_1}.
\end{equation}
The second term can be bounded as 
\begin{align}
\EE[\ell_{1}(H^u, H^y) ] 
&\leq \sum_{r > 0} | \varphi_{r+}(H^u) - \varphi_{r+}(H^y) | \nonumber \\
&\stackrel{(a)}{\leq} 2\sum_{r > 0} | \varphi_{r+}(H^u) - \varphi_{r+}(H^x) | \nonumber \\
&{\leq} 2\sum_{r > 0} | \varphi_{r+}(H^u) - \varphi_{r+}(H^v) | + | \varphi_{r+}(H^x) - \varphi_{r+}(H^v) |, \label{eq:high3}
\end{align}
where $(a)$ follows by Lemma~\ref{lem:round} and the last inequality follows by the triangle inequality. The first term in the last equation corresponds to the smoothing error and we analyze it now. 
\[
\varphi^{v}_{s_{i+1}+} = \sum^{s_{i+1}} _{r=s_i} \varphi^u_r  \frac{(r-s_{i})}{s_{i+1} - s_i} + 
\sum_{r >s_{i+1}}\varphi^u_r.
\]
Since $\varphi^v_j = 0$ for $j \notin S$,
\begin{align*}
\sum^{s_{i+1}-1}_{j=s_i}   \lvert \varphi^v_{j+} - \varphi^u_{j+}\rvert & = \sum^{s_{i+1}-1}_{j=s_i}  \left \lvert  \sum^{s_{i+1}} _{r=s_i} \varphi^u_r  \frac{(r-s_{i})}{s_{i+1} - s_i} -  \sum^{s_{i+1}} _{r=j} \varphi^u_{r} \right \rvert \\
& \leq \sum^{s_{i+1}}_{j=s_i} \sum^{s_{i+1}}_{r=s_{i}} \varphi^u_r \left \lvert \frac{(r-s_{i})}{s_{i+1} - s_i} - \indic_{r \geq j} \right \rvert \\
& =  \sum^{s_{i+1}}_{r=s_i} 2\varphi^u_r\frac{(r-s_{i})(s_{i+1} - r)}{s_{i+1} - s_i} \\
& \leq 2\sum^{s_{i+1}}_{r=s_i}\varphi^u_r  \min(s_{i+1} - r, r-s_{i}).
    \end{align*}
For $r \leq  T'$, that lies between $s_i$ and $s_{i+1}$ in $S$,
\[
|r - s_i| \leq \min(\lfloor T(1+q)^{i+1} \rfloor -r, r- \lfloor T(1+q)^{i}\rfloor) \lesssim  r q.
\]
If $r \geq T'$, the analysis depends on the value of $Z_b$. If $Z_b \geq -M$, $r \geq T'$, and $\varphi^u_r > 0$, then there exists a $s_i$ that is at most $Z^{cl}_i$ away for from $r$.
If not, then the error is at most $2n$. Hence, the smoothing error in expectation can be bounded by 
\begin{equation}
    \label{eq:high5}
\lesssim \sum^{T'}_{r \geq T} \varphi^u_r qr + \sum_{r \geq T'} \varphi^u_r \EE_N[|Z| \indic_{Z_b \geq -M}] + \EE_N[2n \indic_{Z_b < -M}] \lesssim n q  + \frac{n \EE[|Z|]}{T'} + \frac{n}{N}.
\end{equation}
For the second part, 
\begin{align*}
 \sum_{r \geq 1} |\varphi^v_{r+} -\varphi^x_{r+}| 
 & \stackrel{(a)}{\leq} \sum_{s_i \in S} |\varphi^v_{s_i+} -\varphi^x_{s_i+}| (s_{i} - s_{i-1}) \\
 & \stackrel{(b)}{\leq} \sqrt{|S|} \cdot \sqrt{\sum_{s_i \in S}(\varphi^v_{s_i+} -\varphi^x_{s_i+})^2  (s_{i} - s_{i-1})^2 } \\
 & \stackrel{(c)}{\leq} \sqrt{|S|} \cdot \sqrt{\sum_{s_i \in S}(\varphi^w_{s_i+} -\varphi^x_{s_i+})^2 (s_{i} - s_{i-1})^2 },
 \end{align*}
 where $(a)$ follows by observing that $\varphi^x_r = 0$ for $r \notin S$,
$(b)$ follows from the Cauchy-Schwarz inequality. $(c)$ follows from the fact that projecting on to the simplex only increases the error. By the second moments of $Z^s_i$, taking expectation on both sides yields
\begin{align*} 
 \sum_{r \geq 1} \EE_N|\varphi^v_{r+} -\varphi^x_{r+}| \lesssim 
 |S| \cdot \frac{1}{\epsilon}.
\end{align*}
We now bound the size of the set $S$. By step $(2)$ of the algorithm, number of elements in $S$ can be bounded by
\[
\lesssim T + 1 + \log_{1+q} \frac{T'}{T} + \frac{10n}{T'} + U,
\]
where $U$ is the number of elements less than $T'/10$ in $H^{b\ell}$, that upon adding noise increases to $T'$. Expectation of $U$ conditioned on $N$ is 
\begin{equation}
    \label{eq:high7}
\EE_N[U]  \lesssim \left( \frac{n}{T} + M + \EE[|Z|] \right) e^{-9T'\epsilon_2/10} \lesssim
\left( \frac{n}{\sqrt{N\epsilon}} + 1 + \frac{\log N}{\epsilon} \right) e^{-3\sqrt{N/\epsilon}}
\lesssim \frac{n \epsilon}{N} + 1,
\end{equation}
where the second inequality follows by substituting  $T' = 10 \sqrt{N/\epsilon^3_3}$ and third inequality follows by algebraic manipulation. Combining~\eqref{eq:high1},~\eqref{eq:high2},~\eqref{eq:high3},~\eqref{eq:high5}, and~\eqref{eq:high7}, and using the fact that quantization error is at most $\cO(n)$ yields
\begin{align*}
& \lesssim \min(n q, n) + \frac{n \EE[|Z|]}{T'} + \left( T + 1 + \log_{1+q}  \frac{T'}{T} + \frac{n}{T'} + \frac{n\epsilon}{N} \right) \frac{1}{\epsilon}  + \frac{n}{N}\\
& \lesssim \min(n q, n)  + \left( \sqrt{N\epsilon} + \frac{1}{q} \log \frac{2}{\epsilon} + \frac{n\sqrt{\epsilon^3}}{\sqrt{N}} + \frac{n\epsilon}{N}  \right) \frac{1}{\epsilon} + \frac{n}{N} \\
& \lesssim n \min \left( \frac{1}{\sqrt{N}} \sqrt{\frac{1}{\epsilon} \log \frac{2}{\epsilon}}, 1\right)
+ \left(\sqrt{N + 1}  \right) \sqrt{\frac{1}{\epsilon} \log \frac{2}{\epsilon}} + \frac{n\sqrt{\epsilon}}{\sqrt{N}}+ \frac{n}{N} 
 ,
\end{align*}
where the last equation follows by substituting the value of $q = \sqrt{\frac{1}{N\epsilon_3} \log \frac{1}{\epsilon_3}} $.
Taking expectation with respect to $N$ and using Lemma~\ref{lem:geomoments}, \eqref{eq:geomoments1} and the fact that $n\epsilon \gtrsim 1 $ yields,
\[
\EE[\ell_1(h, H^z) \indic_{N > 0}] \lesssim \sqrt{\frac{n}{\epsilon} \cdot \log \frac{2}{\epsilon}} + \frac{1}{\epsilon}.
\]

\subsection{Time complexity}
\label{app:time}

In this section, we provide a proof sketch of the time complexity. 
Suppose $\epsilon > 1$. The number of prevalences in $H^{bs}$ is $T \lesssim \sqrt{N}$. Similarly, the number of counts in $H^{b\ell}$ is $n/T \approx n/\sqrt{N}$. Further, the number of additional fake counts added is $M \lesssim \log N$ and the isotonic regression using PAVA takes linear in the size of the input. Hence, the expected run time, conditioned on $N$ is at most $\sqrt{N} + \frac{n}{\sqrt{N}} + M$, which upon taking expectation w.r.t. $N$ yields an expected run time of $\sqrt{n}$.

If $\epsilon \leq 1$, steps (1)-(4) in \textsc{PrivHist} takes time $\lesssim \sqrt{N/\epsilon} + \log N/\epsilon$. The time complexity to find boundaries, smooth prevalences, add noise, and perform isotonic regression is $|S| + \sqrt{n}$. By the utility analysis, $|S| \lesssim \sqrt{N \log \frac{1}{\epsilon} } + \frac{n\sqrt{\epsilon}}{\sqrt{N}}$. Summing all the time complexities and taking expectation w.r.t. $N$ similar to the utility analysis, yields a total time complexity of $\tcO\left(\sqrt{\frac{n}{\epsilon}} + \frac{\log n}{\epsilon} \right)$.

%% file: lower_bound.tex
\section{Lower bound in the low privacy regime}
\label{app:low_high_privacy}
\begin{lemma}
Let $x \in \{0,1\}^k$. Suppose two vectors $x,y$ are neighbors if $||x-y||_1 \leq 1$. If $\hat{X} = M(x)$ be an $\epsilon$-DP estimate of $x$, then
\[
\max_{x \in \{0,1\}^k} \EE[||x-\hat{X}||_1] \gtrsim  k e^{-\epsilon}.
\]
\end{lemma}
\begin{proof}
Let $p(x)$ be the uniform distribution over $\{0,1\}^k$. For a vector $v \in \{0,1\}^{k-1}$, let $\cX^i_v$ denote the two vectors such that $x^{i-1}_{1} = v^{i-1}_1$ and $x^{k}_{i+1} = v^{k-1}_i$. Then
\begin{align}
    \max_{x \in \{0,1\}^k} \EE[||x-\hat{X}||_1] 
    & \geq \frac{1}{2^k} \sum_{x \in \{0,1\}^k} \EE[||x-\hat{X}||_1] \nonumber \\
    & = \frac{1}{2^k} \sum_{x \in \{0,1\}^k} \sum^k_{i=1}  \EE |x_i -\hat{X}_i|  \nonumber \\
    & = \sum^k_{i=1} \frac{1}{2^k} \sum_{x \in \{0,1\}^k}  \EE |x_i -\hat{X}_i| \nonumber  \\
  & =   \sum^k_{i=1} \frac{1}{2^{k-1}} \sum_{v \in \{0,1\}^{k-1}} \frac{1}{2} \sum_{x \in \cX^i_v} \EE |x_i -\hat{X}_i|. \label{eq:lower1}
\end{align}
For any $v$ and $i$, vectors in $\cX^i_v$ are neighbors and hence by the definition of DP,
\begin{align*}
 \sum_{x \in \cX^i_v} \EE |x_i -\hat{X}_i| 
& = \sum_{\hat{x}_i} \Pr(\hat{X_i} = \hat{x}_i | x_i = 1) | 1- \hat{x}_i|  + \Pr(\hat{X_i} = \hat{x}_i | x_i = 0) | 0 - \hat{x}_i|  \\
& \geq  \sum_{\hat{x}_i}  \Pr(\hat{X_i} = \hat{x}_i | x_i = 1) | 1- \hat{x}_i|  + e^{-\epsilon} \Pr(\hat{X_i} = \hat{x}_i | x_i = 1) | 0 - \hat{x}_i|  \\
& \geq e^{-\epsilon} \sum_{\hat{x}_i}  \Pr(\hat{X_i} = \hat{x}_i | x_i = 1)  \\
& \geq e^{-\epsilon}.
\end{align*}
Substituting the above lower bound in \eqref{eq:lower1} yields the lemma.
\end{proof}
We now use the above bound to prove Theorem~\ref{thm:low_high_privacy}. We first define a of histograms to show the lower bound. Let $k = \sqrt{n}/10$. For a given vector $x \in \{0,1\}^k$, let $\varphi_{4i} = \varphi_{4i+3} = x_{i-1}$ and  $\varphi_{4i+1} = \varphi_{4i+2} =1-  x_{i-1}$, $\varphi_r =1 $ for $r = n - \sum^{4k}_{i=1} \varphi_r r$, and $\varphi_{r} = 0$ otherwise. Observe that $\sum_{r \geq 1}\varphi_r r = n$ and are valid histograms. If two vectors $x$ and $y$ have hamming distance $d$, then the corresponding  distance between anonymized histograms is $2d$.

Consider a slightly different definition of neighboring datasets over histograms, where two datasets are neighboring if the neighboring datasets are distance $2$ apart. If a mechanism is $\epsilon$-DP in the previous definition of neighboring datasets, then the mechanism is $2\epsilon$-DP in the new notion of neighbors. 

One mechanism for releasing the vectors $x \in \{0,1\}^k$ with $2\epsilon$-DP is to encode it as the histograms as mentioned above and release them. Since such a mechanism has hamming distance $\gtrsim ke^{-2\epsilon}$, the anonymized histograms cannot be estimated with accuracy $\gtrsim ke^{-2\epsilon}$ with $2\epsilon$-DP under new definition of neighbors and hence it cannot be estimated with accuracy $\gtrsim ke^{-2\epsilon}$ with $\epsilon$-DP under the old definition of neighbors.

%% file: estimation_and_classification.tex
\section{Proof of Corollary~\ref{cor:properties}}
\label{app:corollary}
Recall that $N$ is the DP estimate of $n$ and $H$ is the DP estimate of $h$. 
In the following Let $X^n$ be the initial set of $n$ samples. Let $X^{N}$ be $N$ samples obtained from $p, X^n$ as follows.
If $N < n$, obtain $X^{N}$ by removing $n - {N}$ samples uniformly from $X^n$.  If ${N} \geq n$, add ${N}-n$ samples from $p$ to $X^n$. Note that $X^{N}$ are $N$ i.i.d. samples from $p$. We bound the error of the estimator as follows.
\[
f(p) - \hat{f}^p = (f(p) - \hat{f}^p)\indic_{N > 0} + (f(p) - \hat{f}^p)\indic_{N = 0}.
\]
Taking expectation on the second term,
\[
\EE[(f(p) - \hat{f}^p)\indic_{N = 0}] = \EE[f(p) \indic_{N = 0}] \leq f(p) e^{-nc\epsilon},
\]
for some constant $c$. For the case $N > 0$, we bound as follows.
\begin{align*}
& f(p) - \hat{f}^p \\
& = f(p) - \sum_{r \geq 1} f(r, {N}) \varphi_r(H) \\
& = f(p) - \sum_{r \geq 1} f(r, {N}) \varphi_r(h(X^{N})) + \sum_{r \geq 1} f(r, {N}) \varphi_r(h(X^N)) 
- \sum_{r \geq 1} f(r, {N}) \varphi_r(h(X^n)) \\
& +\sum_{r \geq 1} f(r, {N}) \varphi_r(h(X^n))  - \sum_{r \geq 1} f(r, {N}) \varphi_r(H).
\end{align*}
We bound each of the three (difference) terms above. First observe that by the assumptions in the theorem:
\[
\EE_N\Bigl[\Bigl|f(p) - \sum_{r \geq 1} f(r, {N}) \varphi_r(h(X^{N}))\Bigr|\Bigr] \leq \cE(\hat{f}, {N}) \leq \cE(\hat{f}, n) + |{N}-n| \frac{{N}^{\beta}}{{N}}.
\]
Since $X^{N}$ and $X^n$ differ in at most $N -n$ terms, by the properties of sorted $\ell_1$ distance,
\[
\sum_{r \geq 1} f(r, {N}) \varphi_r(h(X^N)) 
- \sum_{r \geq 1} f(r, {N}) \varphi_r(h(X^n)) \leq |{N}-n| \max_{r} |f(r,N)| \leq |{N}-n| \cdot \frac{{N}^{\beta}}{{N}}.
\]
Further by the properties of sorted $\ell_1$ distance,
\begin{align*}
\EE_N[|\sum_{r \geq 1} f(r, {N}) \varphi_r(h(X^n))  - \sum_{r \geq 1} f(r, {N}) \varphi_r(H)|]
& \leq \max_{r} |f(r, {N})| \EE_N[\ell_1(h, H)] \\
& \lesssim \frac{N^\beta}{N} \cdot \EE_N[\ell_1(h, H)].
\end{align*}
Summing all the three terms, we get that the error is at most 
\begin{equation}
    \label{eq:sum_three_est}
 \lesssim \cE(\hat{f}, n) + |{N}-n| \cdot \frac{{N}^{\beta}}{{N}} + \frac{N^\beta}{N} \cdot \EE_N[\ell_1(h, H)].
\end{equation}
For $\epsilon > 1$, difference between the expected error and $\cE(\hat{f}, n)$ is 
\begin{align*}
& \lesssim   \EE \left[\frac{N^\beta}{N} |N - n| \indic_{N > 0} + \frac{N^\beta}{N} \left(\sqrt{N} + \frac{n}{\sqrt{N}} \right) \indic_{N > 0} \right] e^{-c\epsilon} \\
&\lesssim   \EE \left[\frac{N^\beta}{N} |N - n| \indic_{N > n/2} + \frac{N^\beta}{N} \left(\sqrt{N} + \frac{n}{\sqrt{N}} \right) \indic_{N > n/2} \right] e^{-c\epsilon} 
+    \EE \left[n \indic_{n/2 \geq N > 0} \right] e^{-c\epsilon} \\
&\lesssim   \EE \left[\frac{n^\beta}{n} |N - n|  + \frac{n^\beta}{\sqrt{n}} \right] e^{-c\epsilon} 
+    \EE \left[n \indic_{n/2 \geq N > 0} \right] e^{-c\epsilon} \\
&\lesssim 
n^{\beta - 1/2} e^{-c\epsilon} + n e^{-nc'\epsilon} \\
& \lesssim 
n^{\beta - 1/2} e^{-c\epsilon},
\end{align*}
where the last inequality uses the fact that $\epsilon > 1$. 
Combining the result with the case $N=0$, we get that
\[
\EE[|f(p) - \hat{f}^p|] \lesssim \cE(\hat{f}, n)+  n^{\beta - 1/2} e^{-c\epsilon} + f(p) e^{-n\epsilon}.
\]
The result for $\epsilon > 1$ follows if
 \[
 n \geq \max\left( n(\hat{f}, \alpha), \cO \left( \left( \frac{1}{\alpha e^{c\epsilon}}
\right)^{\frac{2}{1-2\beta}}+ \frac{1}{\epsilon} \log \frac{f_{\max}}{\alpha} \right) \right).
\]
For $\Omega(1/n) \leq \epsilon \leq 1$, by~\eqref{eq:sum_three_est}, the difference between the expected error and $\cE(\hat{f}, n)$ is 
\begin{align*}
& \lesssim \EE \left[\frac{N^\beta}{N} \left(|N - n| + n \min \left( \frac{1}{\sqrt{N}} \sqrt{\frac{1}{\epsilon} \log \frac{2}{\epsilon}}, 1\right) + \sqrt{N} \sqrt{\frac{1}{\epsilon} \log \frac{2}{\epsilon}} + \frac{n\sqrt{\epsilon}}{\sqrt{N}} + \frac{n}{N} \right) \indic_{N > 0} \right] \\
& \lesssim  \EE \left[\frac{N^\beta}{N} \left(|N - n| + \sqrt{N} \sqrt{\frac{1}{\epsilon} \log \frac{2}{\epsilon}} + \frac{n\sqrt{\epsilon}}{\sqrt{N}} + \frac{n}{N} \right) \indic_{N >n/2} \right]
+ \EE \left[\left(\sqrt{\frac{1}{\epsilon} \log \frac{2}{\epsilon}} + {n} \right) \indic_{n/2\geq N > 0} \right] \\
& \lesssim \EE \left[\frac{n^\beta}{n} |N - n| + \frac{n^\beta}{n} \left(\sqrt{n} \sqrt{\frac{1}{\epsilon} \log \frac{2}{\epsilon}} \right) \right] 
+ \EE \left[\left(\sqrt{\frac{1}{\epsilon} \log \frac{2}{\epsilon}} + {n} \right) \indic_{n/2 \geq N > 0} \right] \\
&  \lesssim\frac{n^\beta}{n\epsilon} + \frac{n^\beta}{\sqrt{n}}  \sqrt{\frac{1}{\epsilon} \log \frac{2}{\epsilon}} 
+ \left(\sqrt{\frac{1}{\epsilon} \log \frac{2}{\epsilon}} + {n} \right) e^{-nc\epsilon} \\
& \lesssim  \frac{n^\beta}{\sqrt{n}}  \sqrt{\frac{1}{\epsilon} \log \frac{2}{\epsilon}} 
+ {n} e^{-nc\epsilon}.
\end{align*}
Combining with the result for the case $N=0$, we get
\[
\EE[|f(p) - \hat{f}^p|] \lesssim \cE(\hat{f}, n)+  n^{\beta - 1/2}  \sqrt{\frac{1}{\epsilon} \log \frac{2}{\epsilon}} + (n + f(p)) e^{-nc\epsilon}.
\]
The result for $\Omega(1/n) \leq \epsilon < 1$ follows if 
 \[
 n \geq  \max\left( n(\hat{f}, \alpha), \cO \left( \left( \frac{\sqrt{\log (2/\epsilon)}}{\alpha \sqrt{\epsilon}}
\right)^{\frac{2}{1-2\beta}} + \frac{1}{\epsilon} \log \frac{f_{\max}}{\alpha\epsilon} \right) \right).
\]